\documentclass{article}

\usepackage{microtype}
\usepackage{graphicx}
\usepackage{subfigure}
\usepackage{booktabs} 
\usepackage[inline]{enumitem}

\usepackage{hyperref}

\usepackage[accepted]{icml2021}

\usepackage{amsmath,amsfonts,bm,mathtools}

\def\eqref#1{Eq.~\ref{#1}}

\def\Algref#1{Algorithm~\ref{#1}}

\def\1{\bm{1}}

\def\vmu{{\bm{\mu}}}

\def\vb{{\bm{b}}}

\def\mA{{\bm{A}}}
\def\mB{{\bm{B}}}
\def\mC{{\bm{C}}}
\def\mD{{\bm{D}}}

\def\mH{{\bm{H}}}
\def\mI{{\bm{I}}}

\def\mK{{\bm{K}}}

\def\mM{{\bm{M}}}

\def\mP{{\bm{P}}}
\def\mQ{{\bm{Q}}}

\def\mU{{\bm{U}}}

\def\mSigma{{\bm{\Sigma}}}

\DeclareMathAlphabet{\mathsfit}{\encodingdefault}{\sfdefault}{m}{sl}
\SetMathAlphabet{\mathsfit}{bold}{\encodingdefault}{\sfdefault}{bx}{n}

\usepackage{amsthm}

\DeclarePairedDelimiter\abs{\lvert}{\rvert}
\DeclarePairedDelimiter{\norm}{\lVert}{\rVert} 
\makeatletter
\let\oldabs\abs
\def\abs{\@ifstar{\oldabs}{\oldabs*}}

\let\oldnorm\norm
\def\norm{\@ifstar{\oldnorm}{\oldnorm*}}
\makeatother

\newcommand{\normdual}[1]{\norm{#1}_{\Lambda_\alpha^*}}
\newcommand{\Da}{D_{\alpha}(\cdot,\cdot)}

\newtheorem{theorem}{Theorem}
\newtheorem{corollary}{Corollary}[theorem]
\newtheorem*{remark}{Remark}

\newtheorem{innercustomlem}{Lemma}
\newenvironment{customlem}[1]
  {\renewcommand\theinnercustomlem{#1}\innercustomlem}
  {\endinnercustomlem}

\newtheorem{innercustomthm}{Theorem}
\newenvironment{customthm}[1]
  {\renewcommand\theinnercustomthm{#1}\innercustomthm}
  {\endinnercustomthm}

\newtheorem{innercustomcor}{Corollary}
\newenvironment{customcor}[1]
  {\renewcommand\theinnercustomcor{#1}\innercustomcor}
  {\endinnercustomcor}

\usepackage{amssymb}
\usepackage{multicol}
\usepackage{bbm}
\usepackage{adjustbox}
\usepackage{multibib}
\newcites{Supp}{Supplement References}

\newtheorem{definition}{Definition}
\newtheorem{lemma}{Lemma}

\icmltitlerunning{Diffusion Earth Mover's Distance and Distribution Embedding}

\begin{document}
\begin{NoHyper}
\twocolumn[
\icmltitle{Diffusion Earth Mover's Distance and Distribution Embeddings}




\icmlsetsymbol{equal}{*}
\icmlsetsymbol{equalsup}{$\dagger$}

\begin{icmlauthorlist}
\icmlauthor{Alexander Tong}{equal,cs}
\icmlauthor{Guillaume Huguet}{equal,dms,mila}
\icmlauthor{Amine Natik}{equal,dms,mila}
\icmlauthor{Kincaid MacDonald}{am}
\icmlauthor{Manik Kuchroo}{gene,cs}
\icmlauthor{Ronald R. Coifman}{am}
\icmlauthor{Guy Wolf}{equalsup,dms,mila}
\icmlauthor{Smita Krishnaswamy}{equalsup,gene,cs}
\end{icmlauthorlist}

\icmlaffiliation{cs}{Dept. of Comp. Sci., Yale University, New Haven, CT, USA;}
\icmlaffiliation{am}{Dept. of Math., Yale University, New Haven, CT, USA;}
\icmlaffiliation{dms}{Dept. of Math. \& Stat., Universit\'{e} de Montr\'{e}al, Montr\'{e}al, QC, Canada;}
\icmlaffiliation{mila}{Mila -- Quebec AI Institute, Montr\'{e}al, QC, Canada;}

\icmlaffiliation{gene}{Department of Genetics, Yale University, New Haven, CT, USA}

\icmlcorrespondingauthor{Smita Krishnaswamy}{smita.krishnaswamy@yale.edu}

\icmlkeywords{Machine Learning, ICML}

\vskip 0.3in
]


\newcommand{\icmlEqualSupContribution}{\textsuperscript{$\dagger$}Equal senior-author contribution.}

\printAffiliationsAndNotice{\icmlEqualContribution; \icmlEqualSupContribution} 

\end{NoHyper}

\begin{abstract}

We propose a new fast method of measuring distances between large numbers of related high dimensional datasets called the Diffusion Earth Mover's Distance (EMD). We model the datasets as distributions supported on common data graph that is derived from the affinity matrix computed on the combined data. In such cases where the graph is a discretization of an underlying Riemannian closed manifold, we prove that Diffusion EMD is topologically equivalent to the standard EMD with a geodesic ground distance. Diffusion EMD can be computed in $\tilde{O}(n)$ time and is more accurate than similarly fast algorithms such as tree-based EMDs. We also show Diffusion EMD is fully differentiable, making it amenable to future uses in gradient-descent frameworks such as deep neural networks. Finally, we demonstrate an application of Diffusion EMD to single cell data collected from 210 COVID-19 patient samples at Yale New Haven Hospital. Here, Diffusion EMD can derive distances between patients on the manifold of cells at least two orders of magnitude faster than equally accurate methods. This distance matrix between patients can be embedded into a higher level patient manifold which uncovers structure and heterogeneity in patients. More generally, Diffusion EMD is applicable to all datasets that are massively collected in parallel in many medical and biological systems.


\end{abstract}

\section{Introduction}

With the profusion of modern high dimensional, high throughput data, the next challenge is the integration and analysis of collections of related datasets. Examples of this are particularly prevalent in single cell measurement modalities where data (such as mass cytometry, or single cell RNA sequencing data) can be collected in a multitude of patients, or in thousands of perturbation conditions~\cite{Shifrut2018}. These situations motivate the organization and embedding of datasets, similar to how we now organize data points into low dimensional embeddings, e.g., with PHATE~\citep{moon_visualizing_2019}, tSNE~\citep{van_der_maaten_visualizing_2008}, or diffusion maps~\citep{coifman_diffusion_2006}). The advantage of such organization is that we can use the datasets as rich high dimensional features to characterize and group the patients or perturbations themselves. In order to extend embedding techniques to entire datasets, we have to define a distance between datasets, which for our purposes are essentially high dimensional point clouds. For this we propose a new form of Earth Mover's Distance (EMD), which we call {\em Diffusion EMD}\footnote{Python implementation is available at \url{https://github.com/KrishnaswamyLab/DiffusionEMD}.}, where we model the datasets as distributions supported on a common data affinity graph.  We provide two extremely fast methods for computing Diffusion EMD based on an approximate multiscale kernel density estimation on a graph. 

Optimal transport is uniquely suited to the formulation of distances between entire datasets (each of which is a collection of data points) as it generalizes the notion of the shortest path between two points to the shortest set of paths between distributions. Recent works have applied optimal transport in the single-cell domain to interpolate lineages~\cite{schiebinger_optimal-transport_2019, yang_scalable_2019, tong_trajectorynet_2020}, interpolate patient states~\cite{tong_interpolating_2020}, integrate multiple domains~\cite{demetci_gromov-wasserstein_2020}, or similar to this work build a manifold of perturbations~\cite{chen_uncovering_2020}. All of these approaches use the standard \textit{primal} formulation of the Wasserstein distance. Using either entropic regularization approximation and the Sinkhorn algorithm~\cite{cuturi_sinkhorn_2013} to solve the discrete distribution case or a neural network based approach in the continuous formulation~\cite{arjovsky_wasserstein_2017}. We will instead use the \textit{dual} formulation through the well-known Kantorovich-Rubinstein dual to efficiently compute optimal transport between many distributions lying on a common low-dimensional manifold in a high-dimensional measurement space. This presents both theoretical and computational challenges, which are the focus of this work.

Specifically, we will first describe a new Diffusion EMD that is an $L^1$ distance between density estimates computed using multiple scales of diffusion kernels over a graph. Using theory on the H\"older-Lipschitz dual norm on continuous manifolds~\cite{leeb_holderlipschitz_2016}, we show that as the number of samples increases, Diffusion EMD is equivalent to the Wasserstein distance on the manifold. This formulation reduces the computational complexity of computing K-nearest Wasserstein-neighbors between $m$ distributions over $n$ points from $O(m^2 n^3)$ for the exact computation to $\tilde{O}(m n)$ with reasonable assumptions on the data. Finally, we will show how this can be applied to embed large sets of distributions that arise from a common graph, for instance single cell datasets collected on large patient cohorts.

Our contributions include:
\begin{enumerate*}
    \item A new method for computing EMD for distributions over graphs called Diffusion EMD. 
    \item Theoretical analysis of the relationship between Diffusion EMD and the snowflake of a standard EMD. 
    \item Fast algorithms for approximating Diffusion EMD.
    \item Demonstration of the differentiability of this framework. 
    \item Application of Diffusion EMD to embedding massive multi-sample biomedical datasets.
\end{enumerate*}


\section{Preliminaries}
\label{sec:Preliminaries}
We now briefly review optimal transport definitions and classic results from diffusion geometry.

\paragraph{Notation.} We say that two elements $A$ and $B$ are equivalent if there exist $c,C>0$ such that $c A\le B \le C A$, and we denote $A\simeq B$. The definition of $A$ and $B$ will be clear depending on the context.

\paragraph{Optimal Transport.} Let $\mu, \nu$ be two probability distributions on a measurable space $\Omega$ with metric $d(\cdot,\cdot)$, $\Pi(\mu, \nu)$ be the set of joint probability distributions $\pi$ on the space $\Omega \times \Omega$, where for any subset $\omega \subset \Omega$, $\pi(\omega \times \Omega) = \mu(\omega)$ and $\pi(\Omega \times \omega) = \nu(\omega)$. The 1-Wasserstein distance $W_d$ also known as the earth mover's distance (EMD) is defined as:
\begin{equation}\label{eq:primal}
    W_d(\mu, \nu) := \inf_{\pi \in \Pi(\mu, \nu)} \int_{\Omega \times \Omega} d(x, y) \pi(dx, dy).
\end{equation}
When $\mu, \nu$ are discrete distributions with $n$ points, then \eqref{eq:primal} is computable in $O(n^3)$ with a network-flow based algorithm~\cite{peyre_computational_2019}.

Let $\| \cdot \|_{L_d}$ denote the Lipschitz norm w.r.t.\ $d$, then the dual of \eqref{eq:primal} is:  
\begin{equation}\label{eq:dual}
    W_d(\mu, \nu) = \sup_{\| f \|_{L_d} \le 1} \int_\Omega f(x) \mu(dx) - \int_\Omega f(y) \nu(dy). 
\end{equation}
This formulation is known as the Kantorovich dual with $f$ as the witness function. Since it is in general difficult to optimize over the entire space of 1-Lipschitz functions, many works optimize the cost over a modified family of functions such as functions parameterized by clipped neural networks~\cite{arjovsky_wasserstein_2017}, functions defined over trees~\cite{le_tree-sliced_2019}, or functions defined over Haar wavelet bases~\cite{gavish_multiscale_2010}.


\paragraph{Data Diffusion Geometry}
Let $(\mathcal{M}, d_\mathcal{M})$ be a connected Riemannian manifold, we can assign to $\mathcal{M}$ sigma-algebras and look at $\mathcal{M}$ as a ``metric measure space''.  We denote by $\Delta$ the Laplace-Beltrami operator on $\mathcal{M}$. For all $x,y \in \mathcal{M}$ let $h_{t}(x, y)$ be the heat kernel, which is the minimal solution of the heat equation:
\begin{equation} \label{eq: heat equation}
\left(\frac{\partial}{\partial t} - \Delta_x \right)h_t = 0,
\end{equation}
with initial condition $\lim_{t \to 0} h_t(x,y)= \delta_y(x)$, where $x \mapsto \delta_y(x)$ is the Dirac function centered at $y$, and $\Delta_x$ is taken with respect to the $x$ argument of $h_t$. Note that as shown in \citet{grigoryan_heat_2014}, the heat kernel captures the local intrinsic geometry of $\mathcal{M}$ in the sense that as $t \to 0$, $\log h_t(x,y) \simeq -d_{\mathcal{M}}^2(x, y)/4t.$ 

Here, in Sec.~\ref{sec:methods:EMD} (Theorem~\ref{th: cor D_a = geo}) we discuss another topological equivalent of the geodesic distance with a diffusion distance derived from the heat operator $\mH_t := e^{-t \Delta}$ that characterizes the solutions of the heat equation (\eqref{eq: heat equation}), and is related to the heat kernel via $\mH_t f = \int h_t(\cdot,y) f(y) dy$~\citep[see ][for further details]{lafferty2005diffusion,coifman_diffusion_2006,grigoryan_heat_2014}. 

It is often useful (particularly in high dimensional data) to consider data as sampled from a lower dimensional manifold embedding in the ambient dimension. This manifold can be characterized by its local structure and in particular, how heat propagates along it. \citet{coifman_diffusion_2006} showed how to build such a propagation structure over discrete data by first building a graph with affinities
\begin{equation} \label{eq: gaussian_kernel}
    (\mK_\epsilon)_{ij} := e^{-\|x_i - x_j\|_2^2 / \epsilon}
\end{equation}
then considering the density normalized operator $\mM_{\epsilon} := \mQ^{-1} \mK_\epsilon \mQ^{-1}$, where $\mQ_{ii} := \sum_j (\mK_{\epsilon})_{ij}$. Lastly, a Markov diffusion operator is defined by
\begin{equation}
\label{eq :P_heat_approx}
    \mP_\epsilon := \mD^{-1}\mM_{\epsilon}, \text{ where } \mD_{ii}:=\sum_j (\mM_{\epsilon})_{ij}.
\end{equation}
Both $\mD$ and $\mQ$ are diagonal matrices. By the law of large numbers, the operator $\mP_\epsilon$ admits a natural continuous equivalent $\tilde{\mP}_\epsilon$, i.e., for $n$ i.i.d.\ points, the sums modulo $n$ converge to the integrals. Moreover, in \citet[Prop. 3]{coifman_diffusion_2006} it is shown that $\lim_{\epsilon\to0}\tilde{\mP}_\epsilon^{t/\epsilon} = e^{-t\Delta} = \mH_t$. In conclusion, the operator $\mP_\epsilon$ converges to $\tilde{\mP}_\epsilon$ as the sample size increases and $\tilde{\mP}_\epsilon$ provide an approximation of the Heat kernel on the manifold. Henceforth, we drop the subscript of $\mP_\epsilon$ to lighten the notation, further we will use the notation $\mP_\epsilon$ for the operator as well as for the matrix (it will be clear in the context).

\section{EMD through the $L^1$ metric between multiscale density estimates}

The method of efficiently approximating EMD that we will consider here is the approximation of EMD through density estimates at multiple scales. Previous work has considered densities using multiscale histograms over images~\citep{indyk_fast_2003}, wavelets over images and trees~\cite{shirdhonkar_approximate_2008, gavish_multiscale_2010} and densities over hierarchical clusters~\cite{le_tree-sliced_2019}. Diffusion EMD also uses a hierarchical set of bins at multiple scales but with smooth bins determined by the heat kernel, which allows us to show equivalence to EMD with a ground distance of the manifold geodesic. These methods are part of a family that we call {\em multiscale earth mover's distances} that first compute a set of multiscale density estimates or histograms where $L^1$ differences between density estimates realizes an effective {\em witness function} and have (varying) equivalence to the Wasserstein distance between the distributions. This class of multiscale EMDs are particularly useful in computing embeddings of distributions where the Wasserstein distance is the ground distance, as they are amenable to fast nearest neighbor queries. We explore this application further in Sec.~\ref{sec:methods:embeddings} and these related methods in Sec.~\ref{sec:supp:multiscale} of the Appendix.

\begin{figure*}[ht]
\begin{center}
\centerline{\includegraphics[width=\linewidth]{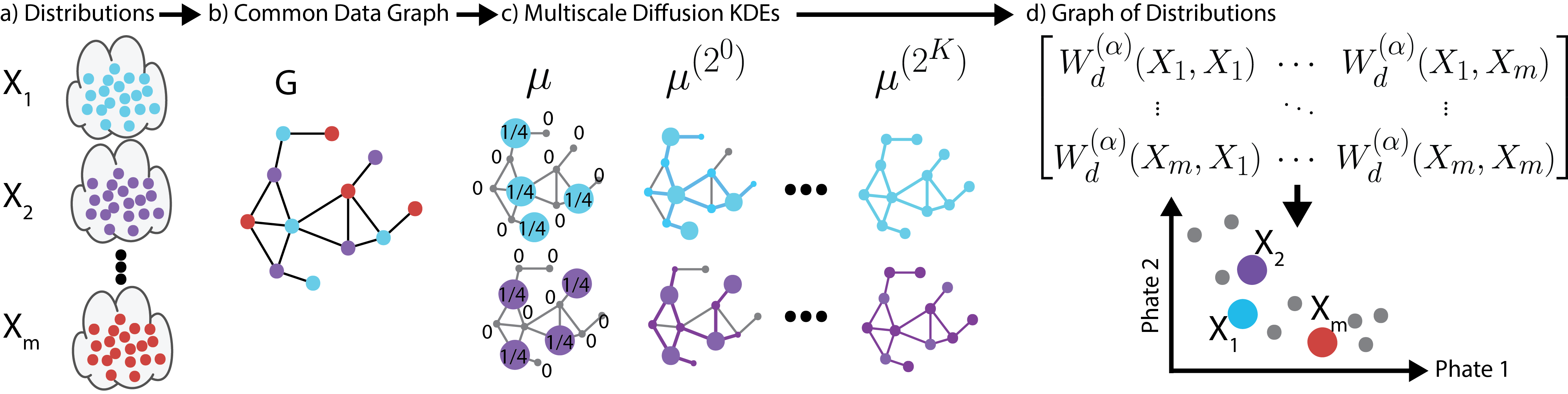}}
\vskip -0.1in
\caption{Diffusion EMD first embeds datasets into a common data graph $G$, then takes multiscale diffusion KDEs for each of the datasets. These multiscale KDEs are then used to compute the Diffusion Earth Mover's Distance between the datasets that can be used in turn to create graphs and embeddings (PHATE \cite{moon_manifold_2018} shown here) of the datasets. }
\label{fig:schematic}
\end{center}
\vskip -0.2in
\end{figure*}

\section{Diffusion Earth Mover's Distance}
We now present the Diffusion EMD, a new Earth Mover's distance based on multiscale diffusion kernels as depicted in Fig.~\ref{fig:schematic}. We first show how to model multiple datasets as distributions on a common data graph and perform multiscale density estimates on this graph.

\subsection{Data Graphs and Density Estimates on Graphs}
\label{sec:datagraphs}

Let $\mathcal{X}=\{X_1, X_2, \ldots, X_m\}$, $\cup_{j=1}^m X_j \subseteq \mathcal{M} \subset \mathbb{R}^d$, be a collection of datasets with $n_i = |X_i|$ and $n = \sum_i n_i$. Assume that the $X_i$'s are independently sampled from a common underlying manifold $(\mathcal{M}, d_{\mathcal{M}})$ which is a Riemannian closed manifold (compact and without boundary) immersed in a (high dimensional) ambient space $\mathbb{R}^d$, with geodesic distance $d_{\mathcal{M}}$. Further, assume that while the underlying manifold is common, each dataset is sampled from a different distribution over it, as discussed below. Such collections of datasets arise from several related samples of data, for instance single cell data collected on a cohort of patients with a similar condition.

Here, we consider the datasets in $\mathcal{X}$ as representing distributions over the common data manifold, which we represent in the finite setting as a common data graph $G_\mathcal{X}=(V,E,w)$ with $V = \cup_{j=1}^m X_j$ and edge weights determined by the Gaussian kernel (see \eqref{eq: gaussian_kernel}), where we identify edge existence with nonzero weights. Then, we associate each $X_i$ with a density measure $\vmu_i^{(t)} : V \to [0,1]$, over the entire data graph. To compute such measures, we first create indicator vectors for the individual datasets on it, let $\bm{1}_{X_i} \in \{0,1\}^n$ be a vector where for each $v \in V, \bm{1}_{X_i}(v) = 1$ if and only if $v \in X_i$. We then derive a kernel density estimate by applying the diffusion operator constructed via \eqref{eq :P_heat_approx} over the graph $G$ to these indicator functions to get scale-dependent estimators
\begin{equation} \label{eq: diffusion_measures}
    \vmu_i^{(t)} := \frac{1}{n_i} \mP^t \bm{1}_{X_i},
\end{equation}
where the scale $t$ is the diffusion time, which can be considered as a meta-parameter \citep[e.g., as used in][]{burkhardt_quantifying_2020} but can also be leveraged in multiscale estimation of distances between distributions as discussed here. Indeed, as shown in \citet{burkhardt_quantifying_2020}, at an appropriately tuned single scale, this density construction yields a discrete version of kernel density estimation.

\subsection{Diffusion Earth Mover's Distance Formulation}
\label{sec:methods:EMD}

We define the Diffusion Earth Mover's Distance between two datasets $X_i, X_j \in \mathcal{X}$ as
\begin{align}\label{eq:diffusion_emd_sum}
    W_{\alpha, K}(X_i, X_j) &:= \sum_{k=0}^{K} \| T_{\alpha, k}(X_i) - T_{\alpha, k}(X_j) \|_1
\end{align}
where $0 < \alpha < 1/2$ is a meta-parameter used to balance long- and short-range distances, which in practice is set close to 1/2, $K$ is the maximum scale considered here, and
\begin{align}\label{eq:diffusion_emd}
    T_{\alpha, k}(X_i) &:= \begin{cases} 2^{-(K-k-1) \alpha} (\vmu_i^{(2^{k+1})} - \vmu_i^{(2^{k})}) & k < K  \\ 
 \vmu_i^{(2^K)} &  k=K 
\end{cases}
\end{align}
Further, to set $K$, we note that if the Markov process governed by $\mP$ converges (i.e., to its stationary steady state) in polynomial time w.r.t.\ $|V|$, then one can ensure that beyond $K = O(\log |V|)$, all density estimates would be essentially indistinguishable as shown by the following lemma, whose proof appears in the Appendix:
\begin{lemma}
There exists a $K = O(\log |V|)$ such that $\vmu_i^{(2^K)} \simeq \mP^{2^K} \bm{1}_{X_i} \simeq \phi_0$ for every $i = 1,\ldots,n$, where $\phi_0$ is the trivial eigenvector of $\mP$ associated with the eigenvalue $\lambda_0 = 1$.
\end{lemma}


To compute the Diffusion EMD $W_{\alpha, K}$ in \eqref{eq:diffusion_emd_sum} involves first calculating diffusions of the datasets $\vmu$, second calculating differences and weighting them in relation to their scale, this results in a vector per distribution of length $O(nK)$, and finally computing the $L^1$ norm between them. The most computationally intensive step, computing the diffusions, is discussed in further detail in Sec.~\ref{sec:methods:computation}.

\subsection{Theoretical Relation to EMD on Manifolds}


We now provide a theoretical justification of the Diffusion EMD defined via \eqref{eq:diffusion_emd_sum} by following the relation established in \citet{leeb_holderlipschitz_2016} between heat kernels and the EMD on manifolds. \citet{leeb_holderlipschitz_2016} define the following ground distance for EMD over $\mathcal{M}$ by leveraging the geometric information gathered from the $L^1$ distances between kernels at different scales.

\begin{definition}
\label{def:diffusion_alpha}
The diffusion ground distance between $x,y \in \mathcal{M}$ is defined as
\begin{equation*}
    D_\alpha(x,y) := \sum_{k\geq 0}2^{-k
    \alpha}\norm{h_{2^{-k}}(x,\cdot)-h_{2^{-k}}(y,\cdot)}_1,
\end{equation*}
for $\alpha\in(0,1/2)$, the scale parameter $K\geq0$ and $h_{t}(\cdot,\cdot)$ the heat kernel on $\mathcal{M}$.
\end{definition}
Note that $\Da$ is similar to the diffusion distance defined in \citet{coifman_diffusion_2006}, which was based on $L^2$ notions rather than $L^1$ here. Further, the following result from \citet[][see Sec. 3.3]{leeb_holderlipschitz_2016} shows that the distance $\Da$ is closely related to the intrinsic geodesic one $d_{\mathcal{M}}(\cdot,\cdot)$.

\begin{theorem}
\label{th: cor D_a = geo}
Let $(\mathcal{M}, d_{\mathcal{M}})$ be a closed manifold with geodesic distance $d_{\mathcal{M}}$, and let $\alpha \in (0,1/2)$. The metric $\Da$ defined via Def.~\ref{def:diffusion_alpha} is equivalent to $d_{\mathcal{M}}(\cdot,\cdot)^{2\alpha}$.
\end{theorem}
The previous theorem justifies why in practice we let $\alpha$ close to $1/2$, because we want the snowflake distance $d_{\mathcal{M}}^{2\alpha}(\cdot, \cdot)$ to approximate the geodesic distance of $\mathcal{M}$. The notion of equivalence established by this result is such that $D_\alpha(x,\cdot)\simeq d(x,\cdot)^{2\alpha}$. It is easy to verify that two equivalent metrics induce the same topology. We note that while here we only consider the Heat kernel, a similar result holds (see Theorem \ref{th: D_alpha = rho} in the Appendix) for a more general family of kernels, as long as they satisfy certain regularity conditions.


For a family of operators $(\mA_t)_{t \in \mathbb{R}^+}$ we define the following metric on distributions; let $\mu$ and $\nu$ be two distributions:
\begin{align} \label{eq: norm_distance_operator}
    \widehat{W}_{\mA_t}(\mu, \nu) & = \norm{\mA_{1} (\mu-\nu)}_1 \\
                                    &+ \sum_{k\geq 0} 2^{-k\alpha}\norm{(\mA_{2^{-(k+1)}}-\mA_{2^{-k}})(\mu-\nu)}_1. \notag
\end{align}
The following result shows that applying this metric for the family of operators $\mH_t$ to get  $\widehat{W}_{\mH_t}$ yields an equivalent of the EMD with respect to the diffusion ground distance $D_{\alpha}(\cdot, \cdot)$.


\begin{theorem}
\label{th: th_heat_was_Da}
The EMD between two distributions $\mu,\nu$ on a closed Riemannian manifold $(\mathcal{M}, d_{\mathcal{M}})$ w.r.t.\ the diffusion ground distance $\Da$, defined via Def.~\ref{def:diffusion_alpha}, given by
\begin{equation}
\label{eq: primal_W_Da}
    W_{D_\alpha}(\mu,\nu) = \inf_{\pi \in \Pi(\mu, \nu)} \int_{\mathcal{M} \times \mathcal{M}} D_\alpha(x,y) \pi(dx, dy),
\end{equation}
is equivalent to $\widehat{W}_{\mH_t}$. That is $W_{D_{\alpha}} \simeq \widehat{W}_{\mH_t}$, where $\mH_t$ is the Heat operator on $\mathcal{M}$.
\end{theorem}

\begin{proof}
In Proposition 15 of \citet{leeb_holderlipschitz_2016}, it is shown that $\mathcal{M}$ is separable w.r.t. $\Da$, hence we can use the Kantorovich-Rubinstein theorem. We let $\Lambda_\alpha$, the space of functions that are Lipschitz w.r.t. $\Da$ and $\normdual{\cdot}$, the norm of its dual space $\Lambda_\alpha^*$. The norm is defined by
\begin{equation*}
    \normdual{T} := \sup_{\norm{f}_{\Lambda_\alpha}\leq 1 }\int_\mathcal{M}f dT.
\end{equation*}
In Theorem 4 of \citet{leeb_holderlipschitz_2016}, it is shown that both $W_{D_{\alpha}}$ and $\widehat{W}_{\mH_t}$ are equivalent to the norm $\normdual{\cdot}$.
\end{proof}

We consider the family of operators $(\mP^{t/\epsilon}_{\epsilon})_{t \in \mathbb{R}^+}$, which is related to the continuous equivalent of the stochastic matrix defined in \eqref{eq :P_heat_approx}. In practice, we use this family of operators to approximate the heat operator $\mH_t$. Indeed, when we take a small value of $\epsilon$, as discussed in Sec.~\ref{sec:Preliminaries}, we have from~\citet{coifman_diffusion_2006} that this is a valid approximation. 

\begin{corollary}
\label{th: cor_W_using_P_to_H}
Let $\mP_\epsilon$ be the continuous equivalent of the stochastic matrix in \eqref{eq :P_heat_approx}. For $\epsilon$ small enough, we have:
\begin{equation} \label{eq:lim_aprox_to_heat}
    \widehat{W}_{\mP_\epsilon^{t/\epsilon}} \simeq W_{D_\alpha}.
\end{equation}
\end{corollary}
\eqref{eq:lim_aprox_to_heat} motivates our use of \eqref{eq:diffusion_emd_sum} to compute the Diffusion EMD. The idea is to take only the first $K$ terms in the infinite sum $\widehat{W}_{\mP_{\epsilon}^{t/\epsilon}}$ and then choosing $\epsilon := 2^{-K}$ would give us exactly \eqref{eq:diffusion_emd_sum}. We remark that the summation order of \eqref{eq:diffusion_emd_sum} is inverted compared to \eqref{eq:lim_aprox_to_heat}, but in both cases the largest scale has the largest weight. Finally, we state one last theorem that brings our distance closer to the Wasserstein w.r.t.\ $d_{\mathcal{M}}(\cdot,\cdot)$; we refer the reader to the Appendix for its proof.

\begin{theorem}
\label{th: W_DA equiv W_MIN}
Let $\alpha \in (0,1/2)$ and $(\mathcal{M}, d_{\mathcal{M}})$ be a closed manifold with geodesic $d_{\mathcal{M}}$. The Wasserstein distance w.r.t.\ the diffusion ground distance $D_\alpha(\cdot, \cdot)$ is equivalent to the Wasserstein distance w.r.t\ the snowflake distance $d_{\mathcal{M}}(\cdot, \cdot)^{2\alpha}$ on $\mathcal{M}$, that is $W_{D_\alpha} \simeq W_{d_{\mathcal{M}}^{2\alpha}}.$
\end{theorem}


\begin{corollary}
For each $1 \leq i,j \leq m$, let $X_i, X_j \in \mathcal{X}$ be two datasets with size $n_i$ and $n_j$ respectively, and let $\mu_i$ and $\mu_j$ be the continuous distributions corresponding to the ones of $X_i$ and $X_j$, let $K$ be the largest scale and put $N =\min(K, n_i,n_j)$. Then, for sufficiently big $N \to \infty$ (implying sufficiently small $\epsilon = 2^{-K} \to 0$):
\begin{equation}
    W_{\alpha, K}(X_i, X_j) \simeq W_{d_{\mathcal{M}}^{2\alpha}}(\mu_i, \mu_j),
\end{equation}
for all $\alpha \in (0, 1/2)$.
\end{corollary}
In fact we can summarize our chain of thought as follows
\begin{align*}
     W_{\alpha,K}(X_i,X_j) &\stackrel{{(a)}}{\simeq} \widehat{W}_{\mP_{\epsilon}^{t/\epsilon}}(\mu_i, \mu_j)   &\stackrel{\textit{(b)}}{\simeq} \widehat{W}_{\mH_t}(\mu_i, \mu_j)   \\
    &\stackrel{\textit{(c)}}{\simeq} W_{D_\alpha}(\mu_i, \mu_j)                                &\stackrel{\textit{(d)}}{\simeq} W_{d_{\mathcal{M}}^{2\alpha}}(\mu_i, \mu_j),
\end{align*}
where the approximation \textit{(a)} is due to the fact that the discrete distributions on $X_i$ and $X_j$ converge respectively to $\mu_i$ and $\mu_j$ when $\min(n_i, n_j) \to \infty$. Further, $W_{\alpha, K}(X_i, X_j)$ approximate the infinite series $\widehat{W}_{\mP_{\epsilon}^{t/\epsilon}}(\mu_i, \mu_j)$ as in~\eqref{eq: norm_distance_operator} when $K\to \infty$, note also that we take $\epsilon = 2^{-K}$ so that the largest scale in~\eqref{eq:diffusion_emd_sum} is exactly $2^K$. The approximation in \textit{(b)} comes from the approximation of the heat operator as in \citet{coifman_diffusion_2006}, \textit{(c)} comes from Theorem~\ref{th: th_heat_was_Da} and \textit{(d)} comes from Theorem~\ref{th: cor D_a = geo}.

\subsection{Efficient Computation of Dyadic Scales of the Diffusion Operator}\label{sec:methods:computation}

The most computationally intensive step of Diffusion EMD requires computing dyadic scales of the diffusion operator $\mP$ times $\vmu$ to estimate the density of $\vmu^{(t)}$ at multiple scales which we will call $\vb$. After this embedding, Diffusion EMD between two embeddings $\vb_i, \vb_j$ is computed as $| \vb_i - \vb_j |$, i.e. the $L^1$ norm of the difference. Computing the embedding $\vb$ naively by first powering $\mP$ then right multiplying $\vmu$, may take up to $2^K$ matrix multiplications which is infeasible for even moderately sized graphs. We assume two properties of $\mP$ that makes this computation efficient in practice. First, that $\mP$ is sparse with order $\tilde{O}(n)$ non-zero entries. This applies when thresholding $\mK_\epsilon$ or when using a $K$-nearest neighbors graph to approximate the manifold~\cite{van_der_maaten_visualizing_2008, moon_visualizing_2019}. Second, that $\mP^t$ is low rank for large powers of $t$.

While there are many ways to approximate dyadic scales of $\mP$, we choose from two methods depending on the number of distributions $m$ compared to the number of points in the graph $n$. When $m \ll n$, we use a method based on Chebyshev approximation of polynomials of the eigenspectrum of $\mP$ as shown in Alg.~\ref{alg:cheb_embedding}. This method is efficient for sparse $\mP$ and a small number of distributions~\cite{shuman_chebyshev_2011}. For more detail on the error incurred by using Chebyshev polynomials we refer the reader to~\citet[Chap. 3]{trefethen_approximation_2013}. In practice, this requires for the approximating polynomial of $J$ terms, computation of $mJ$ (sparse) matrix vector multiplications for a worst case time complexity of $O(J m n^3)$, but in practice is $\tilde{O}(J m n)$ where $J$ is a small constant (see Fig.~\ref{fig:ablation}(e)). However, while asymptotically efficient, when $m \gg n$ in practice this can be inefficient as it requires many multiplications of the form $\mP \vmu$. 

\begin{algorithm}[tb]
    \caption{Chebyshev embedding }
    \label{alg:cheb_embedding}
\begin{algorithmic}
    \STATE {\bfseries Input:} $n \times n$ graph kernel $\mK$, $n \times m$ distributions $\vmu$, maximum scale $K$, and snowflake constant $\alpha$.
    \STATE {\bfseries Output:} $m \times (K + 1) n$ distribution embeddings $\vb$
    \STATE $\mQ \leftarrow Diag(\sum_i \mK_{ij})$
    \STATE $\mK^{norm} \leftarrow \mQ^{-1} \mK \mQ^{-1}$
    \STATE $\mD \leftarrow Diag(\sum_i \mK^{norm}_{ij})$
    \STATE $\mM \leftarrow \mD^{-1/2}\mK^{norm}\mD^{-1/2}$
    \STATE $\mU \mSigma \mU^T = \mM$; $\quad \mU$ orthogonal, $\mSigma$ Diagonal
    \STATE $\vmu^{(2^0)} \leftarrow \mP \vmu \leftarrow \mD^{-1/2} \mM \mD^{1/2} \vmu$
    \FOR{$k=1$ {\bfseries to} $K$}
        \STATE $\vmu^{(2^k)} \leftarrow \mP^{2^k}\vmu \leftarrow \mD^{-1/2} \mU (\mSigma)^{2^k} \mU^T \mD^{1/2} \vmu$
        \STATE $\vb_{k-1} \leftarrow 2^{(K-k-1)\alpha}(\vmu^{(2^k)} - \vmu^{(2^{k-1})})$
    \ENDFOR
    \STATE $\vb_{K} \leftarrow \vmu^{(2^K)}$
    \STATE $\vb \leftarrow [\vb_0, \vb_1, \ldots, \vb_K]$
\end{algorithmic}
\end{algorithm}

In the case where $m \gg n$, approximating powers of $\mP$ and applying these to the $n \times m$ collection of distributions $\vmu$ once is faster. This method is also useful when the full set of distributions is not known and can be applied to new distributions one at a time in a data streaming model. A naive approach for computing $\mP^{2^K}$ would require $K$ dense $n \times n$ matrix multiplications. However, as noted in \citet{coifman_diffusion_2006-1}, for higher powers of $\mP$, we can use a much smaller basis. We use an algorithm based on interpolative decomposition~\cite{liberty_randomized_2007, bermanis_multiscale_2013} to reduce the size of the basis, and subsequently the computation time, at higher scales. In this algorithm we first determine the approximate rank of $\mP^{2^k}$ using an estimate of the density of the eigenspectrum of $\mP$ as in~\citet{dong_network_2019}. We then alternate steps of downsampling the basis to the specified rank of $\mP^{2^k}$ with (randomized) interpolative decomposition with steps of powering $\mP$ on these bases. Informally, the interpolative decomposition selects a representative set of points that approximate the basis well. In the worst case this algorithm can take $\tilde{O}(m n^3)$ time to compute the diffusion density estimates, nevertheless with sparse $\mP$ with a rapidly decaying spectrum, this algorithm is $\tilde{O}(m n)$ in practice. For more details see Alg.~\ref{alg:id_embedding} and Sec.~\ref{sec:supp:algorithm} of the Appendix.

\subsection{Subsampling Density Estimates}\label{sec:methods:subsampling}

The density estimates created for each distribution are both large and redundant with each distribution represented by a vector of $(K+1) \times n$ densities. However, as noted in the previous section, $\mP^{2^k}$ can be represented on a smaller basis, especially for larger scales. Intuitively, the long time diffusions of nodes that are close to each other are extremely similar. Interpolative decomposition~\cite{liberty_randomized_2007, bermanis_multiscale_2013} allows us to pick a set of points to center our diffusions kernels such that they approximately cover the graph up to some threshold on the rank. In contrast, in other multiscale EMD methods the bin centers or clusters are determined randomly, making it difficult to select the number of centers necessary. Furthermore, the relative number of centers at every scale is fixed, for example, Quadtree or Haar wavelet based methods~\cite{indyk_fast_2003, gavish_multiscale_2010} use $2^{d^k}$ centers at every scale, and a clustering based method~\cite{le_tree-sliced_2019} selects $C^k$ clusters at every scale for some constant $C$. Conversely, in Diffusion EMD, by analyzing the behavior of $\mP^{2^k}$, we intelligently select the number of centers needed at each scale based on the approximate rank of $\mP^{2^k}$ up to some tolerance at each scale. This does away with the necessity of a fixed ratio of bins at every scale, allowing adaptation depending on the structure of the manifold and can drastically reduce the size representations (see Fig.~\ref{fig:ablation}(c)). For the Chebyshev polynomials method, this subsampling is done post computation of diffusion scales, and for the method based on approximating $\mP^{2^k}$ directly the subsampling happens during computation. To this point, we have described a method to embed distributions on a graph into a set of density estimates whose size depends on the data, and the spectrum decay of $\mP$. We will now explore how to use these estimates for exploring the Diffusion EMD metric between distributions.

\subsection{Diffusion EMD Based Embeddings of Samples}\label{sec:methods:embeddings}

Our main motivation for a fast EMD computed on related datasets is to examine the space of the samples or datasets themselves, i.e., the higher level manifold of distributions. In terms of the clinical data, on which we show this method, this would be the relationship between patients themselves, as determined by the EMD between their respective single-cell peripheral blood datasets.  Essentially, we create a kernel matrix $\mK_{\mathcal{X}}$ and diffusion operator $\mP_{\mathcal{X}}$ between datasets where the samples are nodes on the associated graph. This diffusion operator $\mP_{\mathcal{X}}$ can be embedded using diffusion maps~\cite{coifman_diffusion_2006} or visualized with a method like PHATE \cite{moon_manifold_2018} that collects the information into two dimensions as shown in Sec.~\ref{sec:results}. We note this higher level graph can be a sparse KNN graph, particularly given that a diffusion operator on the graph can allow for global connections to be reformed via $t$-step path probabilities.  

Multiscale formulations of EMD as in \eqref{eq:diffusion_emd} are especially effective when searching for nearest neighbor distributions under the Wasserstein metric~\cite{indyk_fast_2003, backurs_scalable_2020} as this distance forms a normed space, i.e., a space where the metric is induced by the $L_1$ norm of the distribution vectors and their differences.  Data structures such as kd-trees, ball trees, locality sensitive hashing, are able to take advantage of such normed spaces for  sub-linear neighbor queries. This is in contrast to network-flow or Sinkhorn type approximations that require a scan through all datapoints for each nearest neighbor query as this metric is not derived from a norm. 


\subsection{Gradients of the Earth Mover's Distance}\label{sec:methods:gradients}

One of the hindrances in the use of optimal transport-based distances has been the fact that it cannot be easily incorporated into deep learning frameworks. 
Gradients with respect to the EMD are usually found using a trained Lipschitz discriminator network as in Wasserstein-GANs~\cite{arjovsky_wasserstein_2017}, which requires unstable adversarial training, or by taking derivatives through a small number of iterations of the Sinkhorn algorithm~\cite{frogner_learning_2015,BTSSPP15, genevay_learning_2018, liu_learning_2020}, which scales with $O(n^2)$ in the number of points. Tree based methods that are linear in the number of points do not admit useful gradients due to their hard binning over space, giving a gradient of zero norm almost everywhere. 

We note that, given the data diffusion operator, the computation of Diffusion EMD is differentiable and, unlike Tree-based EMD, has smooth bins and therefore a non-zero gradient norm near the data (as visible in Fig.~\ref{fig:1d}). Further, computation of the gradient only requires powers of the diffusion operator multiplied by the indicator vector describing the distribution on the graph. In fact, as mentioned in the supplementary material (Sec.~\ref{subsec:gradient}) for each $v \in V$, the gradient of the Diffusion EMD $\partial W_{\alpha,K}(X_i, X_j)/\partial v$ depends mainly on the gradients $\partial \mP_{\epsilon}^{2^k}/\partial v$ for $0 \leq K$ which can be expressed in terms of the gradient of the Gaussian kernel $ \partial \mK_{\epsilon}/\partial v$. This last quantity is easy to compute. In Sec.~\ref{subsec:gradient} of the Appendix, we give an exact process on computing the gradient of the Diffusion EMD.



\section{Results}\label{sec:results}

In this section, we first evaluate the Diffusion EMD on two manifolds where the ground truth EMD with a geodesic ground distance is known, a swiss roll dataset and spherical MNIST~\citep{cohen_convolutional_2017}. On these datasets where we have access to the ground truth geodesic distance we show that Diffusion EMD is both faster and closer to the ground truth than comparable methods. Then, we show an application to a large single cell dataset of COVID-19 patients where the underlying metric between cells is thought to be a manifold~\cite{moon_manifold_2018,kuchroo_multiscale_2020}. We show that the manifold of patients based on Diffusion EMD by capturing the graph structure, better captures the disease state of the patients.

\paragraph{Experimental Setup.} 
We consider four baseline methods for approximating EMD: $\text{QuadTree}(D)$~\cite{backurs_scalable_2020} which partitions the dataspace in half in each dimension up to some specified depth $D$, $\text{ClusterTree}(C, D)$~\cite{le_tree-sliced_2019} which recursively clusters the data with $C$ clusters up to depth $D$ using the distance between clusters to weight the tree edges, and the convolutional Sinkhorn distance~\cite{solomon_convolutional_2015} with the same graph as used in Diffusion EMD. QuadTree and ClusterTree are fast to compute. However, because they operate in the ambient space, they do not represent the geodesic distances on the manifold in an accurate way. The convolutional Sinkhorn method represents the manifold well but is significantly slower even when using a single iteration. For more details on related work and the experimental setup see Sections \ref{sec:supp:related} and \ref{sec:supp:experiments} of the Appendix respectively. 

\begin{figure}[ht]
\begin{center}
\centerline{\includegraphics[width=\columnwidth]{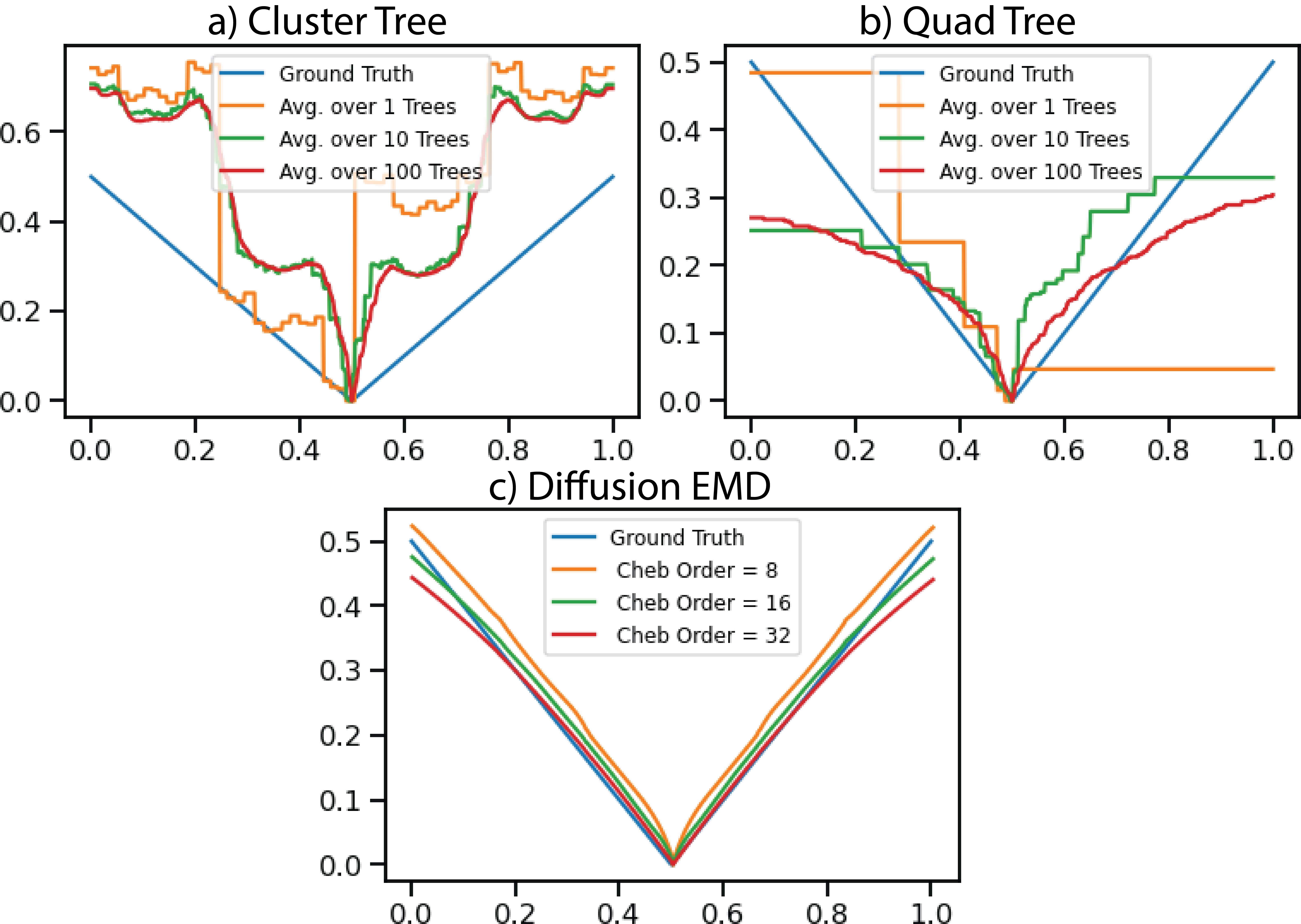}}
\vskip -0.1in
\caption{Wasserstein distance of indicator distributions $\bm{1}_x, x \in [0,1]$ from $\bm{1}_{0.5}$ computed using linear EMD methods $L^2$ distance: (a) ClusterTree (b) QuadTree and (c) Diffusion EMD.}
\label{fig:1d}
\end{center}
\vskip -0.2in
\end{figure}
\paragraph{1D data.} We first illustrate the advantages of using Diffusion EMD over tree-based methods (ClusterTree and QuadTree) on a line graph with 500 points spaced in the interval $[0,1]$. In Fig.~\ref{fig:1d} we depict the Wasserrstein distance between an indicator function at each of these 500 points ($\mathbf{1}_{x}$) and an indicator function at $x=0.5$. In the case of indicator functions the Wasserstein distance is exactly the ground distance. We show three approximations of each method, varying the number of trees and the Chebyshev polynomial order respectively. It is clear that Diffusion EMD achieves a much better approximation of the ground truth primarily due to its use of smooth bins.

\begin{figure}[ht]
\begin{center}
\centerline{\includegraphics[width=\columnwidth]{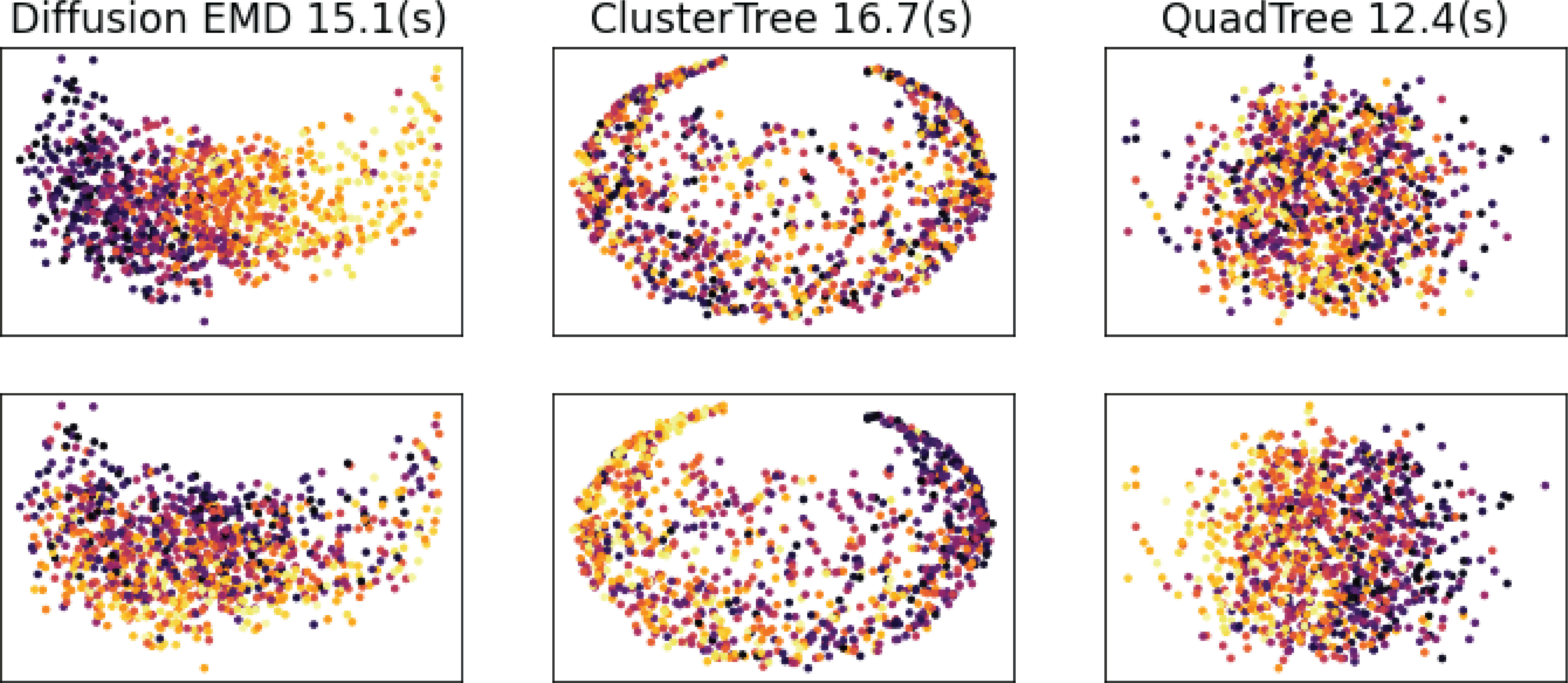}}
\vskip -0.1in
\caption{Swiss roll dataset embeddings of $m=1000$ distributions with $n=10,000$ total points rotated into 10D colored by ground truth 2D sheet axes. Diffusion EMD recreates the manifold better in similar time.}
\label{fig:swissroll_embed}
\end{center}
\vskip -0.2in
\end{figure}
\begin{figure}[ht]
\begin{center}
\centerline{\includegraphics[width=\columnwidth]{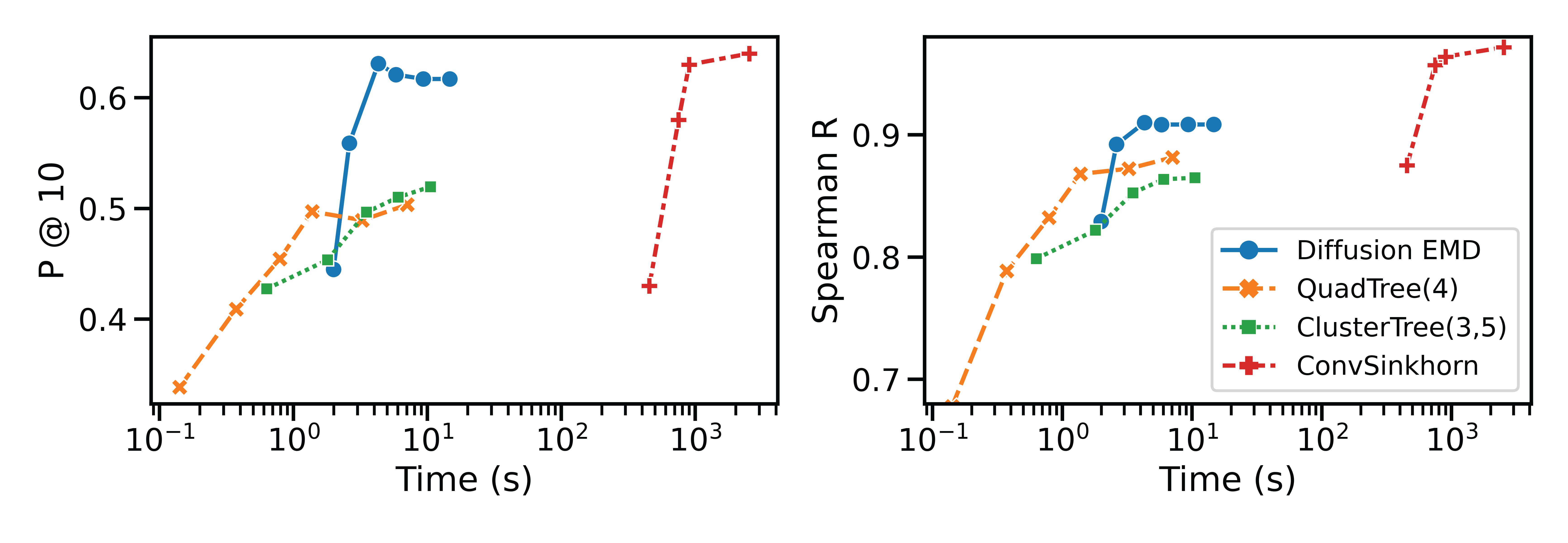}}
\vskip -0.1in
\caption{Accuracy of methods measured via P@10 (left) and Spearman coefficient (right), against their (log scaled) computation time in seconds on the swiss roll dataset. Variations of methods are over Chebyshev approximation order for Diffusion EMD, \# of trees for tree methods, and number of iterations for conv. Sinkhorn. Diffusion EMD is more accurate than tree methods and orders of magnitude faster than conv. Sinkhorn even with a single iteration.}
\label{fig:swissroll}
\end{center}
\vskip -0.2in
\end{figure}
\begin{figure}[ht]
\begin{center}
\centerline{\includegraphics[width=\columnwidth]{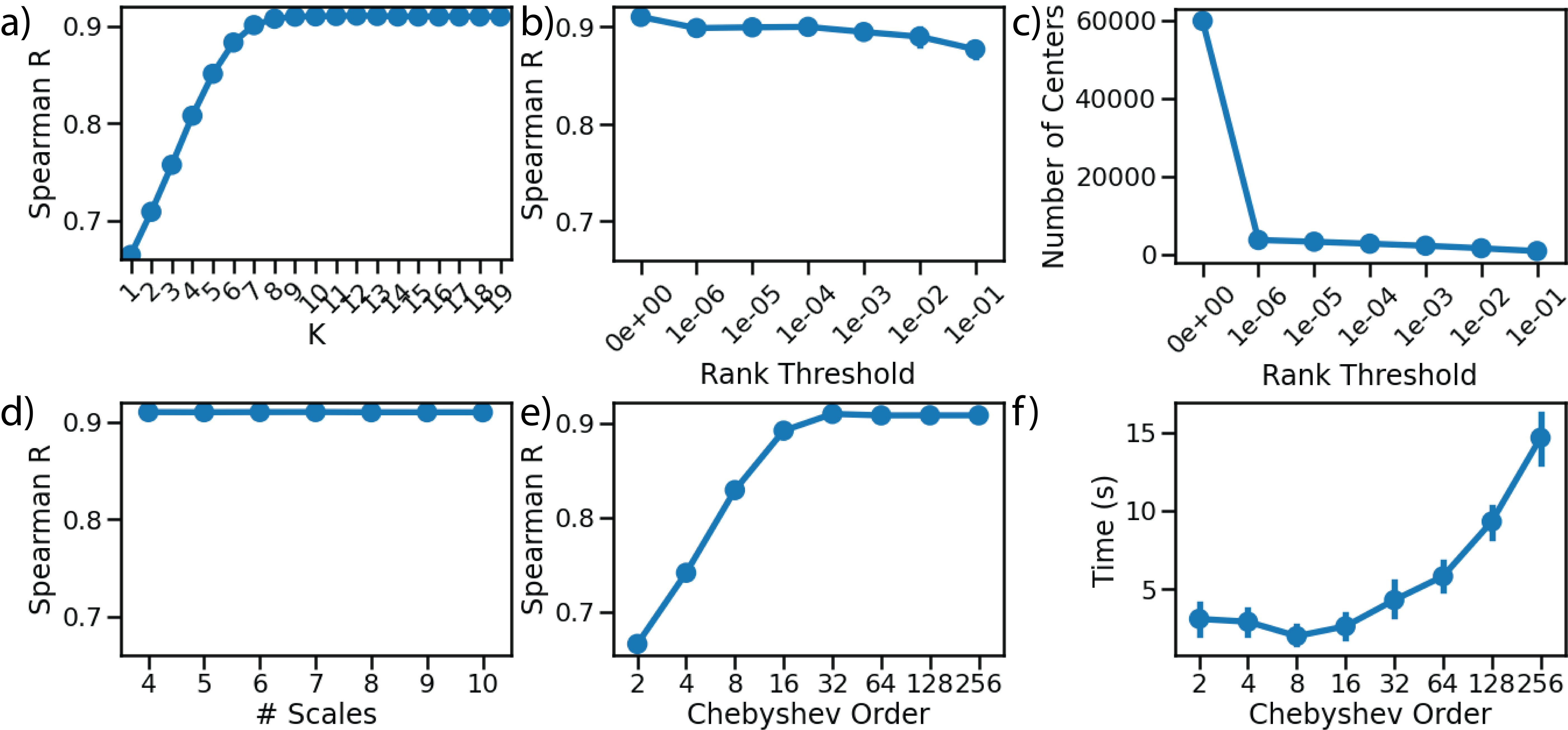}}
\vskip -0.1in
\caption{Ablation study of major parameters for Chebyshev polynomial approximation on the swiss roll dataset. Mean and std.\ over 10 runs over (a) values of the maximum scale $K$, (b) the rank threshold in interpolative decomposition, (c) the total number of centers in the $L^1$ representation which drops with decomposition, (d-e) performance against the \# of scales, and the order of the polynomial, both are very stable after a certain point, and (f) time vs. the Chebyshev order. }
\label{fig:ablation}
\end{center}
\vskip -0.2in
\end{figure}
\paragraph{Swiss roll data.} The next application we explore is to a dataset where we have distributions on a manifold for which the geodesic distance is easily computable. In this way we can compare to the ground truth EMD between distributions. We generate $m=100$ Gaussians on the swiss roll with 100 points each for a total of $n=10,000$ points on the graph. We compare each method on two metrics, the 10-nearest neighbor accuracy measured (P@10) where a P@10 of 1 means that the 10-nearest neighbors are the same as the ground truth. We compare the rankings of nearest neighbors over the entire dataset using the Spearman-$\rho$ correlation coefficient, which measures the similarity of nearest neighbor rankings. This coefficient ranges between -1 for inversely ranked lists and 1 for the same ranks. This measures rankings over the entire dataset equally rather than only considering the nearest neighbors. Visually, we show embeddings of the swiss roll in Fig.~\ref{fig:swissroll_embed}, where the 2D manifold between distributions is best captured by Diffusion EMD given a similar amount of time. 

In Fig.~\ref{fig:swissroll}, we investigate the time vs. accuracy tradeoff of a number of fast EMD methods on the swiss roll. We compare against the ground truth EMD which is calculated with the exact EMD on the ``unrolled'' swiss roll in 2D. We find that Diffusion EMD is more accurate than tree methods for a given amount of time and is much faster than the convolutional Sinkhorn method and only slightly less accurate. To generate multiple models for each dataset we vary the number of trees for tree methods, the Chebyshev order for Diffusion EMD, and the number of iterations for convolutional Sinkhorn. We search over and fix other parameters using a grid search as detailed in Sec.~\ref{sec:supp:experiments} of the Appendix.

In Fig.~\ref{fig:ablation}, we vary parameters of the Chebyshev polynomial algorithm of Diffusion EMD. Regarding performance, we find Diffusion EMD is stable to the number of scales chosen after a certain minimum maximum scale $K$, Chebyshev polynomial order, and the number of scales used. By performing interpolative decomposition with a specified rank threshold on $\mP^{2^k}$ we can substantially reduce the embedding size at a small cost to performance Fig.~\ref{fig:ablation}(b,c).

\paragraph{Spherical MNIST.}
\vskip -0.2in
\begin{table}[tbh]
\caption{Classification accuracy, P@10, Spearman $\rho$ and runtime (in minutes) on 70,000 distributions from Spherical MNIST.}
\label{tables:spherical_mnist}
\begin{center}
\begin{small}
\begin{sc}
\adjustbox{width=\linewidth}{
\begin{tabular}{lrrrr}
    \toprule
 & Accuracy & P@10 & Spearman $\rho$ & Time \\ 
    \midrule
        Diff. EMD    & \textbf{95.94}    & \textbf{0.611}    & \textbf{0.673}  & 34m            \\ 
        Cluster  & 91.91             & 0.393             & 0.484           & 30m            \\ 
        Quad     & 79.56             & 0.294             & 0.335           & \textbf{16m}   \\ 
    \bottomrule
\end{tabular}

}\end{sc}
\end{small}
\end{center}
\vskip -0.1in
\end{table}

Next, we use the Spherical MNIST dataset to demonstrate the efficacy of the interpolative decomposition based approximation to Diffusion EMD, as here $m \gg n$. Each image is treated as a distribution (of pixel intensities) over the sphere. To evaluate the fidelity of each embedding, we evaluate the 1-NN classification on the embedding vectors in addition to P@10 and Spearman coefficient in Tab.~\ref{tables:spherical_mnist}. Diffusion EMD creates embeddings that better approximate true EMD over the sphere than tree methods in a similar amount of time, which in this case also gives better classification accuracy. 

\paragraph{Single cell COVID-19 patient data.}
\begin{figure}[ht]
\begin{center}
\centerline{\includegraphics[width=\columnwidth]{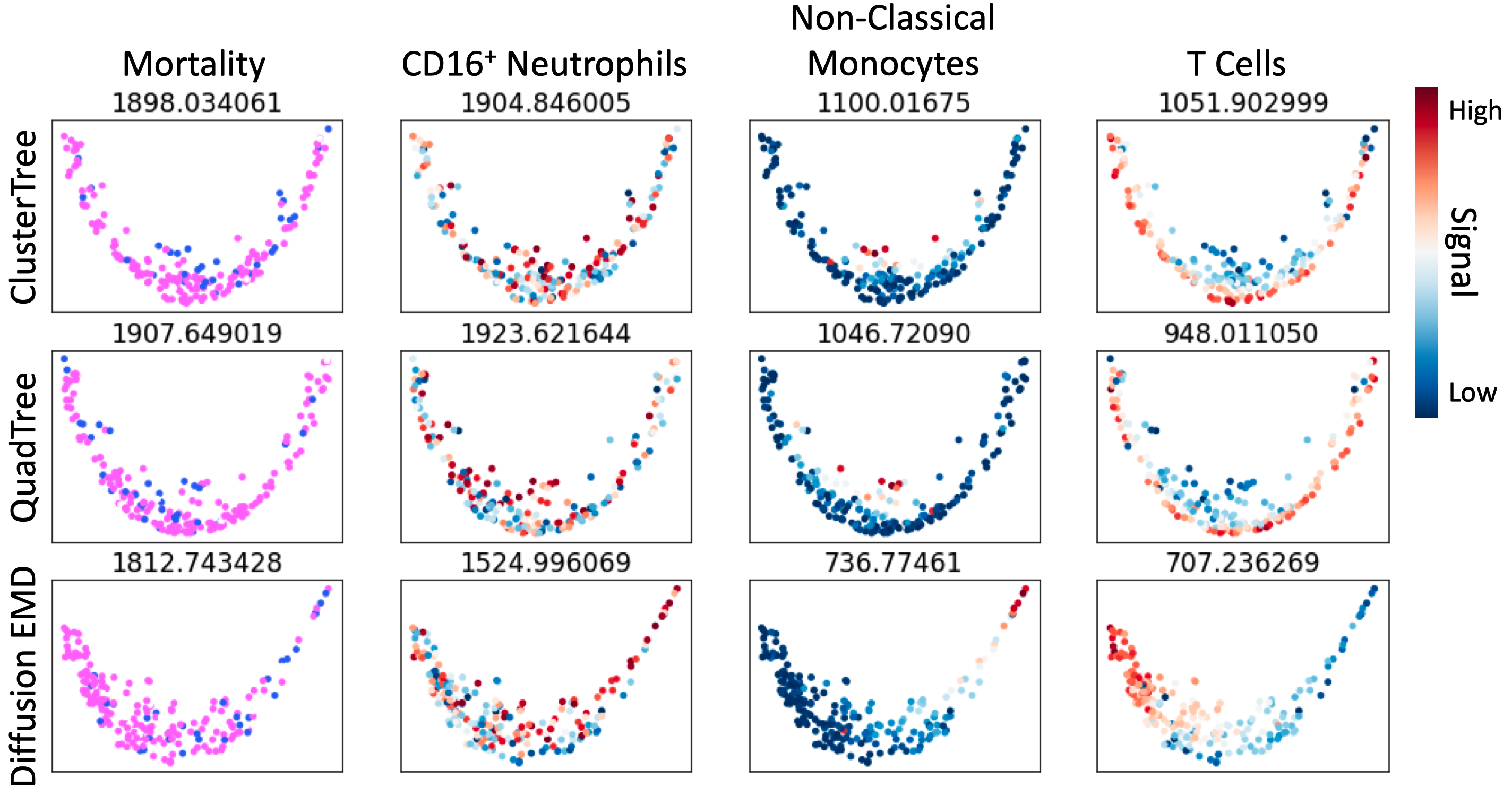}}
\caption{Embedding of 210 patients through different manifold constructions.  Visualizing patient eventual mortality and cell types predictive of disease outcome on each manifold. Laplacian smoothness reported on each signal for each manifold.}
\label{fig:covid}
\end{center}
\vskip -0.2in
\end{figure}

The COVID-19 pandemic has driven biologists to generate vast amounts of cellular data on hospitalized patients suffering from severe disease. A major question in clinicians minds is determining a priori which patients may be at risk for worse outcomes, including requiring increased ventilatory support and increased risk of mortality.  Certain cell types found in the blood, such as CD16$^{+}$ Neutrophils, T cells and non-classical monocytes, have been associated with and predictive of mortality outcome. Ideally, a manifold of patients would find these cellular populations to occupy a region of high mortality for CD16$^{+}$ Neutrophils and non-classical monocytes and low mortality for T cells.  In order to construct a manifold of patients suffering from COVID-19, we analyzed 210 blood samples from 168 patients infected with SARS-CoV-2 measured on a myeloid-specific flow cytometry panel, an expanded iteration of a previously published dataset~\cite{Lucas2020}. We embedded 22 million cells from these patients into a common combined cell-cell graph with 27,000 nodes as defined in Sec.~\ref{sec:datagraphs}. We then computed Diffusion EMD and other methods on these datasets. Diffusion EMD is computed by using indicator vectors for each patient converted to density estimates as in \eqref{eq: diffusion_measures}.

On an informative embedding of patients, similar patients, with similar features (such as mortality) would localize on the manifold and thus the important features should be smooth over the manifold. Furthermore, cell types which are correlated with outcome either positively or negatively should also be smooth and either correlated or anticorrelated with outcome. To quantify this, we compute the smoothness with respect to the patient manifold by using a {\em Laplacian quadratic form} with respect to the 10-NN graph between patients. Convolutional Sinkhorn does not scale to this data, so we compare a patient manifold created with Diffusion EMD to ones created with QuadTree and ClusterTree. Diffusion EMD is able to use the manifold of cells where QuadTree and ClusterTree are built in the ambient space. In Fig.~\ref{fig:covid} we visualize relevant signals over the patients using PHATE~\cite{moon_visualizing_2019} overlayed with the quadratic smoothness of the signal over the graph. While the mortality signal appeared enriched in the right branches of both the Diffusion EMD and QuadTree manifolds, it did not localize as well on the ClusterTree manifold. Both CD16$^{+}$ Neutrophils and non-classical monocytes appeared smoother over the Diffusion EMD manifold than the comparison manifolds. Since both cell types are associated with mortality, it was interesting to see them both enriched in high mortality region of the Diffusion EMD manifold but not the others. Finally, T cells, which are negatively correlated with mortality appeared smoothly enriched in the Diffusion EMD manifold in a region with no mortality. In QuadTree and ClusterTree constructions, T cells appeared enriched throughout the manifold, no localizing smoothly to a region with low mortality.  These experiments show that the patient manifold constructed with Diffusion EMD is more informative, smoothly localizing key signals, such as patient outcome and predictive cell types. For a more details see Sec.~\ref{sec:supp:experiments} of the Appendix.

\section{Conclusion}

In this work we have introduced Diffusion EMD, a multiscale distance that uses heat kernel diffusions to approximate the earth mover's distance over a data manifold. We showed how Diffusion EMD can efficiently embed many samples on a graph into a manifold of samples in $\tilde{O}(mn)$ time more accurately than similarly efficient methods. This is useful in the biomedical domain as we show how to embed COVID-19 patient samples into a higher level patient manifold that more accurately represents the disease structure between patients. Finally, we also show how to compute gradients with respect to Diffusion EMD, which opens the possibility of gradients of the earth mover's distance that scale linearly with the dataset size.

\section*{Acknowledgements}
This research was partially funded by IVADO PhD Excellence Scholarship [\emph{A.N.}]; IVADO Professor startup \& operational funds, IVADO Fundamental Research Proj.\ grant PRF-2019-3583139727, Canada CIFAR AI Chair [\emph{G.W.}]; Chan-Zuckerberg Initiative grants 182702 \& CZF2019-002440 [\emph{S.K.}]; and NIH grants R01GM135929 \& R01GM130847 [\emph{G.W., S.K.}]. The content provided here is solely the responsibility of the authors and does not necessarily represent the official views of the funding agencies.

\bibliography{clean}
\bibliographystyle{icml2021}
\clearpage 

\appendix

\onecolumn

\section*{Supplemental Material}
We first analyze the theoretical framework of Diffusion EMD in Appendix~\ref{sec:supp:theory}. Next we discuss related worki n Appendix~\ref{sec:supp:related}, with a particular focus on multiscale methods for EMD. In Appendix~\ref{sec:supp:algorithm} we provide further detail on the two algorithms for computing Diffusion EMD, the first based on Chebyshev approximation, and the second based on directly powering the diffusion operator while reducing the basis. We discuss gradients of Diffusion EMD in section \ref{subsec:gradient}. Finally we provide further experimental details in Appendix~\ref{sec:supp:experiments}. 

\section{General framework and proofs}\label{sec:supp:theory}

\subsection{General framework}

We now recall some useful results from \citet{leeb_holderlipschitz_2016}. We will only present an overview, for a rigorous exposition, we suggest \citet{leeb_topics_2015}, \citet{leeb_holderlipschitz_2016} and \citet{grigoryan_heat_2015}.

We let $\mathcal{Z}$ be a sigma-finite measure space of dimension $n$, $d(\cdot,\cdot)$ its intrinsic distance and $\xi$ its associated sigma-finite measure. In order to evaluate the EMD between two measures on $\mathcal{Z}$, one would need to know $d(\cdot,\cdot)$. Either directly, by using equation \ref{eq:primal}, or implicitly, to define the space of Lipschitz functions with respect to $d(\cdot,\cdot)$ in equation \ref{eq:dual}. One of the contributions of this paper is to define a kernel based metric $D_\alpha(\cdot,\cdot)$, where the kernel depends on $\alpha\in(0,1)$, and show that the metrics $D_\alpha(\cdot,\cdot)$ and $d(\cdot,\cdot)$ are closely related. Next, the objective is to use $D_\alpha(\cdot,\cdot)$ as the ground distance to compute the EMD in its dual form. That is a norm between two distributions acting on the space of Lipschitz functions with respect to $D_\alpha(\cdot,\cdot)$. To do so, the authors define $\Lambda_\alpha$ the space of functions that are Lipschitz with respect to $D_\alpha(\cdot,\cdot)$, and its dual $\Lambda^*_\alpha$; the space of measures acting on $f\in\Lambda_\alpha$. Further, in order to use the Kantorovich–Rubinstein theorem, they show that the space $\mathcal{Z}$ is separable with respect to the metric $D_\alpha(\cdot,\cdot)$. Thus, the norm of the dual space $\normdual{\cdot}$ can be used to compute the EMD with $D_\alpha(\cdot,\cdot)$ as the ground distance. Lastly, they define two norms $\normdual{\cdot}^{(1)}$ and $\normdual{\cdot}^{(2)}$ that are equivalent to $\normdual{\cdot}$ on $\Lambda^*$. In practice, these norms are much faster to compute. 

The metric $D_\alpha(\cdot,\cdot)$ is defined using a family of kernels $\{a_t(\cdot,\cdot)\}_{t\in\mathbb{R}^+}$ on $\mathcal{Z}$. For each kernel, we define an operator $\mA_t$ as $(\mA_tf)(x) = \int_\mathcal{Z} a_t(x,y)f(y)d\xi(y)$. These kernels must respect some properties:
\begin{itemize}
    \item The semigroup property: for all $s,t>0, \mA_t \mA_s = \mA_{t+s}$;
    \item The conservation property: $\int_\mathcal{Z} a_t(x,y)d\xi(y) = 1$;
    \item The integrability property: there exists $C>0$ such that $\int_\mathcal{Z} \abs{a_t(x,y)}d\xi(y) < C$, for all $t>0$ and $x\in\mathcal{Z}$.
\end{itemize}
Considering only the dyadic times, that is $t = 2^{-k}$, we define the kernel $p_k(\cdot,\cdot) := a_{2^{-k}}(\cdot,\cdot)$ and the operator $\mP_k = \mA_{2^{-k}}$, for all $k\in\mathbb{N}$. By leveraging the local geometric information gathered from the $L^1$ distance between two measures
\begin{equation*}
    D_k(x,y) := \norm{p_k(x,\cdot)-p_k(y,\cdot)}_1,
\end{equation*} 
the authors define the following multiscale metric
\begin{equation*}
    D_\alpha (x,y) := \sum_{k\geq 0}2^{-k\alpha}D_k(x,y).
\end{equation*}
To interpret $D_k(\cdot,\cdot)$ in an intuitive way, consider the case where $a_t(\cdot,\cdot)$ defines a random walk on $\mathcal{Z}$. As a consequence, $a_t(x,B(y,r))$ is the probability to move from $x$ to a point in $B(y,r)$ in $t$ steps. Moreover, for any $x\in\mathcal{Z}$, $a_t(x,\cdot)$ defines a distribution, therefore $D_k(x,y)$ is the $L^1$ distance between two distributions induced by the points $x$ and $y$. Since these distributions depend on the number of steps $t$, it is clever to consider a distance that includes many scales, just like $D_\alpha(\cdot,\cdot)$. Another property needs to be verified by the kernel $a_t(\cdot,\cdot)$, this property depends on the distance $D_\alpha(\cdot,\cdot)$. Namely, the geometric property: there exist $C>0$ and $\alpha\in(0,1)$ such that for all $k\in\mathbb{N}$ and $x\in\mathcal{Z}$
\begin{equation*}
    \int_\mathcal{Z}\abs{p_k(x,y)}D_\alpha(x,y)d\xi(y) \leq C 2^{-k\alpha}.
\end{equation*}

We need to add three stronger regularity conditions on the kernel $a_t(\cdot,\cdot)$, for $D_\alpha(\cdot,\cdot)$ to be closely related to the intrinsic distance $d(\cdot,\cdot)$:
\begin{enumerate}
    \item An upper bound on the kernel: there exist a non-negative, monotonic decreasing function $\phi : \mathbb{R}^+ \to \mathbb{R}$ and $\beta>0$ such that for any $\gamma<\beta$, the function verifies 
    \begin{equation*}
        \int_0^\infty \tau^{n+\gamma-1}\phi(\tau)d\tau
    \end{equation*}
    and
    \begin{equation*}
        \abs{a_t(x,y)}\leq t^{-\frac{n}{\beta}}\phi\left( \frac{d(x,y)}{t^{1/\beta}} \right).
    \end{equation*}
    \item Hölder continuity estimate: there exist $\Theta>0$ sufficiently small, such that, for all $t\in(0,1]$ and all $x,y\in\mathcal{Z}$, the distance verifies $d(x,y)\leq t^{1/\beta}$ and for all $u\in\mathcal{Z}$ the difference between kernels is bounded
    \begin{equation*}
        \abs{a_t(x,u)-a_t(y,u)}\leq t^{-\frac{n}{\beta}} \left(\frac{d(x,y)}{t^{1/\beta}}\right)^\Theta \!\!\phi\left( \frac{d(x,y)}{t^{1/\beta}} \right).
    \end{equation*}
    \item A local lower bound: there exist a monotonic decreasing function $\psi:\mathbb{R}^+\to\mathbb{R}$ and $R>0$, such that, for all $t\in(0,1]$ and all $x,y$ where $d(x,y)<R$, we have
    \begin{equation*}
        \abs{a_t(x,y)}\geq t^{-\frac{n}{\beta}}\psi\left( \frac{d(x,y)}{t^{1/\beta}} \right).
    \end{equation*}
\end{enumerate}
It is shown that the heat kernel on a closed Riemannian manifold respects all these conditions  \citep[][see Sec. 3.3]{leeb_holderlipschitz_2016}.

\begin{definition}
A distance $d(\cdot,\cdot)^b$ is a snowflake of a distance $d(\cdot,\cdot)$ if $b\in(0,1)$. Moreover, the Hölder space is the space of functions that are Lipschitz continuous w.r.t.\ a snowflake distance, hence the terminology used in \citet{leeb_holderlipschitz_2016}. 
\end{definition}

We now state an important theorem from \citet[Thm. 2]{leeb_holderlipschitz_2016}.
\begin{theorem}
\label{th: D_alpha = rho}
Consider a sigma-finite measure space $\mathcal{Z}$ of dimension $n$ with metric $d(\cdot,\cdot)$ and a measure $\xi$ such that $\xi(B(x,r))\simeq r^n$. If the family of kernels $\{a_t(\cdot,\cdot)\}_{t\in\mathbb{R}^+}$ respect the condition 1,2 and 3, then, for $0<\alpha<\min(1,\Theta/\beta)$, the distance $D_\alpha(\cdot,\cdot)$ is equivalent to the thresholded snowflake distance $\min[1,d(\cdot,\cdot)^{\alpha\beta}]$.
\end{theorem}

\begin{remark}
In our case, because we used the heat kernel on a closed Riemannian manifold $\mathcal{M}$, we had $\Da \simeq d_{\mathcal{M}}(\cdot,\cdot)^{2\alpha}$ (Thm. \ref{th: cor D_a = geo}). This can be justified from \citet[Cor. 2]{leeb_holderlipschitz_2016}. Using the same notation as in the corollary, we define $C:=\max_{x,y} d_{\mathcal{M}}(x,y)^{2\alpha}$ (which is finite due to the assumptions on $\mathcal{M}$), thus we can bound the constant $B$
\begin{equation*}
    B = \frac{B}{d_{\mathcal{M}}(x,y)^{2\alpha}}d_{\mathcal{M}}(x,y)^{2\alpha} \ge \frac{B}{C}d_{\mathcal{M}}(x,y)^{2\alpha}.
\end{equation*}
\end{remark}

The previous theorem closely links the two considered distances. It also motivated the goal to compute the EMD w.r.t.\ $D_\alpha(\cdot,\cdot)$. In \citet{leeb_holderlipschitz_2016}[Prop. 15], it is shown that $\mathcal{Z}$ is a separable space with respect to the metric $D_\alpha(\cdot,\cdot)$. Hence, we can use the Kantorovich-Rubinstein theorem to express the EMD in its dual form.

First, we need to define the space of functions that are Lipschitz with respect to $D_\alpha(\cdot,\cdot)$. For a fix $\alpha\in(0,1)$, we note this space by $\Lambda_\alpha$. It corresponds to the set of functions $f$ on $\mathcal{Z}$ such that the norm 
\begin{equation*}
    \norm{f}_{\Lambda_\alpha} := \sup_x\abs{f(x)} + \sup_{x\neq y}\frac{\abs{f(x)-f(y)}}{D_\alpha(x,y)}
\end{equation*}
is finite. Next, we need to define the space dual to $\Lambda_\alpha$, which is noted by $\Lambda_\alpha^*$. For a function $f\in\Lambda_\alpha$ and a $L^1$ measure $T$, we define $\langle f , T\rangle := \int_\mathcal{Z} f dT$. The space $\Lambda_\alpha^*$ is the space of $L^1$ measure with the norm 
\begin{equation*}
    \normdual{T} := \sup_{\norm{f}_{\Lambda_\alpha}\leq 1 }\langle f,T\rangle.
\end{equation*}
In practice, this norm would still be computationally expensive. However, the authors show that the norms
\begin{align*}
    \normdual{T}^{(1)}&:= \norm{P_0^* T}_1 + \sum_{k\geq 0} 2^{-k\alpha}\norm{(P^*_{k+1}-P^*_k)T}_1 \\
    \normdual{T}^{(2)}&:= \norm{P_0^* T}_1 + \sum_{k\geq 0} 2^{-k\alpha}\norm{(P^*_{k}-P^*_0)T}_1
\end{align*}
are equivalent to the norm $\normdual{\cdot}$ on $\Lambda_\alpha^*$ \citep[][Thm. 4]{leeb_holderlipschitz_2016}. Where $P^*_k$ is the adjoint of $P_k$. Finally, using the Kantorovich-Rubinstein theorem, we get
\begin{equation*}
    \inf_{\pi \in \Pi(\mu, \nu)} \int_{\mathcal{Z} \times \mathcal{Z}} D_\alpha(x,y) \pi(dx, dy) \normdual{\mu-\nu} \simeq\normdual{\mu-\nu}^{(1)}\simeq\normdual{\mu-\nu}^{(2)},
\end{equation*}
where $D_\alpha(\cdot,\cdot) \simeq \min[1,d(\cdot,\cdot)^{\beta\alpha}]$ and $d(\cdot,\cdot)$ is the ground distance of $\mathcal{Z}$. 

In conclusion, using $\normdual{\cdot}^{(1)}$ or $\normdual{\cdot}^{(2)}$ yields a norm equivalent to $\normdual{\cdot}$. The norm $\normdual{\mu-\nu}$ is equal to the Wasserstein distance between the distributions $\mu$ and $\nu$ with respect to the ground distance $D_\alpha(\cdot,\cdot)$.

\subsection{Proofs of section 4.2}

\begin{customlem}{1}
Assuming that $\mP$ converges (i.e. to its stationary distribution) in polynomial time w.r.t. $\abs{V}$, then there exists a $K = O(\log |V|)$ such that $\vmu_i^{(2^K)} \simeq \mP^{2^K} \bm{1}_{X_i} \approx \phi_0$ for every $i = 1,\ldots,n$, where $\phi_0$ is the trivial eigenvector of $\mP$ associated with the eigenvalue $\lambda_0 = 1$.
\end{customlem}

\begin{proof}
First we notice that $\mP$ is reversible with respect to $\bm{\pi}_i = \mD_{ii}/\sum_i\sum_j\mD_{ij}$. Since $\mP$ is ergodic it converges to its unique stationary distribution $\bm{\pi}$. Moreover, we assumed this convergence to be in polynomial time w.r.t. $\abs{V}$, i.e. we define 
\begin{equation*}
   \Delta_i(k) := \frac{1}{2}\sum_{j} \abs{\mP^{2^k}_{ij}-\bm{\pi}_j}
\end{equation*}
and the mixing time for a $\epsilon>0$
\begin{equation*}
    \tau_i(\epsilon) := \min\{k:\Delta_i(k^\prime)<\epsilon , \,\forall\, k^\prime \ge k\}.
\end{equation*}
Then, by our assumption, there exist $2^K$ = $\tau_i(\epsilon) = O(|V|)$, thus $K = O(\log |V|)$. Intuitively, for all $k\ge K$, each row of the matrix $\mP^{2^k}$ is approximately equal to $\bm{\pi}$ (w.r.t. $\epsilon$), as a consequence $\mP^{2^k} \bm{1}_{X_i} \approx \phi_0$.

\end{proof}

\begin{customthm}{3}
Let $\alpha \in (0,1/2)$ and $(\mathcal{M}, d_{\mathcal{M}})$ be a closed manifold with geodesic $d_{\mathcal{M}}$. The Wasserstein distance w.r.t.\ the diffusion ground distance $D_\alpha(\cdot, \cdot)$ is equivalent to the Wasserstein distance w.r.t\ the snowflake distance $d_{\mathcal{M}}(\cdot, \cdot)^{2\alpha}$ on $\mathcal{M}$, that is $W_{D_\alpha} \simeq W_{d_{\mathcal{M}}^{2\alpha}}.$
\end{customthm}
\begin{proof}
We will prove this for a more general framework, we let two metrics $d_1(\cdot,\cdot)$ and $d_2(\cdot,\cdot)$ on a sigma-finite measure space $\mathcal{Z}$, and let $\mu$ and $\nu$ be two measures. We assume $d_1 \simeq d_2$, that is there exist two constants $c, C > 0$ such that for all $x, y\in\mathcal{Z}$  we have $c\, d_1(x,y)\leq d_2(x,y)\leq C\, d_1(x,y)$. Using the same notation as in \eqref{eq:primal}, for all $\pi \in \Pi(\mu, \nu)$ we have 
\begin{align*}
    c \int_{\mathcal{Z}\times\mathcal{Z}}d_1(x,y)\pi(dx,dy) \le \int_{\mathcal{Z}\times\mathcal{Z}} \,d_2(x,y)\pi(dx,dy)
    \le \, C\int_{\mathcal{Z}\times\mathcal{Z}}d_1(x,y)\pi(dx,dy)
\end{align*}
then 
\begin{align*}
    c \inf_{\pi \in \Pi(\mu, \nu)}\int_{\mathcal{Z}\times\mathcal{Z}}d_1(x,y)\pi(dx,dy) \le \inf_{\pi \in \Pi(\mu, \nu)} \int_{\mathcal{Z}\times\mathcal{Z}} \,d_2(x,y)\pi(dx,dy)
    \le \, C \inf_{\pi \in \Pi(\mu, \nu)} \int_{\mathcal{Z}\times\mathcal{Z}}d_1(x,y)\pi(dx,dy)
\end{align*}
which is the same as $c W_{d_1}(\mu, \nu) \leq W_{d_2}(\mu, \nu) \leq C W_{d_1}(\mu, \nu)$, this proves that whenever $d_1$ and $d_2$ are equivalent then $W_{d_1}$ and $W_{d_2}$ are equivalent as well.
\end{proof}

\begin{customcor}{2.1}
Let $\mP_\epsilon$ be the continuous equivalent of the stochastic matrix in \eqref{eq :P_heat_approx}. For $\epsilon$ small enough, we have:
\begin{equation} \label{eq:supp:lim_aprox_to_heat}
    \widehat{W}_{\mP_\epsilon^{t/\epsilon}} \simeq W_{D_\alpha}.
\end{equation}

\end{customcor}
\begin{proof}
Note that by Theorem~\ref{th: th_heat_was_Da} it is equivalent to show that $\widehat{W}_{\mP_\epsilon^{t/\epsilon}} \simeq \widehat{W}_{\mH_t}$. First note that according to \cite{coifman_diffusion_2006} we have $\norm{\mP_\epsilon^{t/\epsilon} - \mH_t}_{L^2(\mathcal{M})} \to 0$ as $\epsilon \to 0$, note that since $\mathcal{V}\mathrm{ol}(\mathcal{M}) < \infty$ (closed Manifold) and according to Cauchy-Schwarz inequality we get $\norm{\mP_\epsilon^{t/\epsilon} - \mH_t}_{L^1(\mathcal{M})} \to 0$ as $\epsilon \to 0$. Generally speaking, convergence in $L^2$ implies convergence in $L^1$ when the manifold is closed. Now put $\mD_{\epsilon, t} = \mP_\epsilon^{t/\epsilon} - \mH_t$, then we have $\norm{\mD_{\epsilon, t}}_{L^1(\mathcal{M})} \to 0$ as $\epsilon \to 0$. Let $\mu$ and $\nu$ be two measures and let $\delta > 0$, choose $\epsilon>0$ small enough so that $\norm{\mD_{t, \epsilon} \gamma}_1 < \delta \norm{\gamma}_1$ for all $t>0$, where  $\gamma = \mu - \nu$,
\begin{align*}
    \widehat{W}_{\mD_{\epsilon, t}}(\mu, \nu) 
    &= \left\| \mD_{\epsilon, 1} \gamma \right\|_1 + \sum_{k \geq 0} 2^{-k \alpha} \left\|\left(\mD_{\epsilon, 2^{-(k+1)}} - \mD_{\epsilon, 2^{-k}}\right)\gamma\right\|_1 \\
    &\leq \left\| \mD_{\epsilon, 1} \gamma \right\|_1 + \sum_{k \geq 0} 2^{-k \alpha} \left\|\mD_{\epsilon, 2^{-(k+1)}} \gamma \right\|_1 + 2^{-k \alpha} \left\|\mD_{\epsilon, 2^{-k}}\gamma\right\|_1 \\
    &\leq \delta \norm{\gamma}_1 + 2 \, \delta \, \norm{\gamma}_1  \sum_{k \geq 0} 2^{-k\alpha} \\ 
    & = \delta \norm{\gamma}_1 \left(1 + \frac{2}{1 - 2^{-\alpha}}\right).
\end{align*}
This proves that for all $t > 0$, for all $\delta > 0$ there exists an $\epsilon > 0$ sufficiently small such that
\begin{equation*}
     \widehat{W}_{\mD_{\epsilon, t}}(\mu, \nu)  \leq \delta \norm{\mu - \nu}
\end{equation*}
Note also that using the reverse triangle inequality we can easily show that $\left| \widehat{W}_{\mP_{\epsilon}^{t/\epsilon}}(\mu, \nu) - \widehat{W}_{\mH_{t}}(\mu, \nu)\right| \leq \widehat{W}_{\mD_{\epsilon, t}}(\mu, \nu)$, let $\delta > 0$ then for  $\epsilon > 0$ small enough we get
\begin{equation*}
   \widehat{W}_{\mH_{t}}(\mu, \nu) - \delta \norm{\mu - \nu}_1 \leq  \widehat{W}_{\mP_{\epsilon}^{t/\epsilon}}(\mu, \nu)  \leq \widehat{W}_{\mH_{t}}(\mu, \nu) + \delta \norm{\mu - \nu}_1,
\end{equation*}
according to \citeSupp{wang1997sharp} we can get lower bounds of the heat kernel (which implies lower bounds for the heat operator), we have $\widehat{W}_{\mH_{t}}(\mu, \nu) \geq \norm{ \mH_{1} (\mu - \nu) }_1 \geq C \norm{\mu - \nu}_1$ for some $C > 0$. If $\delta < C/2 $ then for sufficiently small $\epsilon > 0$,
\begin{equation*}
    \frac{1}{2}\widehat{W}_{\mH_{t}}(\mu, \nu) \leq  \widehat{W}_{\mP_{\epsilon}^{t/\epsilon}}(\mu, \nu)  \leq \frac{3}{2} \widehat{W}_{\mH_{t}}(\mu, \nu)
\end{equation*}
which completes the proof.
\end{proof}

\begin{customcor}{3.1}
For each $1 \leq i,j \leq m$ let $X_i, X_j \in \mathcal{X}$ be two datasets with size $n_i$ and $n_j$ respectively, and let $\mu_i$ and $\mu_j$ be the continuous distributions corresponding to the ones of $X_i$ and $X_j$, let $K$ be the largest scale and put $N =\min(K, n_i,n_j)$. Then, for sufficiently big $N \to \infty$ (implying sufficiently small $\epsilon = 2^{-K} \to 0$):
\begin{equation}
    W_{\alpha, K}(X_i, X_j) \simeq W_{d_{\mathcal{M}}^{2\alpha}}(\mu_i, \mu_j),
\end{equation}
for all $\alpha \in (0, 1/2)$.
\end{customcor}
\begin{proof}
First define
\begin{equation}\label{eq: W_discrete_measures}
    W_{\alpha, K, \epsilon}(X_i, X_j) := \norm{\bm{\mu}_i^{(1/\epsilon)} - \bm{\mu}_j^{(1/\epsilon)} }_1 + \sum_{k = 0}^{K - 1}2^{-k\alpha}\norm{(\bm{\mu}_{i}^{(2^{-(k+1)}/\epsilon)} - \bm{\mu}_{i}^{(2^{-k}/\epsilon)}) - (\bm{\mu}_{j}^{(2^{-(k+1)}/\epsilon)} - \bm{\mu}_{j}^{(2^{-k}/\epsilon)})}_1,
\end{equation}
note that for $\epsilon = 2^{-K}$ we have $W_{\alpha, K, \epsilon}(X_i, X_j) = W_{\alpha, K}(X_i, X_j)$. Since $n_i \to \infty$ and $n_j \to \infty$, then using Monte-Carlo integration we get
\begin{equation*}
    \lim_{n_i, n_j \to \infty} W_{\alpha, K, \epsilon}(X_i, X_j) = W_{\alpha, K, \epsilon}(\mu_i, \mu_j)
\end{equation*}
where $W_{\alpha, K, \epsilon}(\mu_i, \mu_j)$ has the same expression as in~\eqref{eq: W_discrete_measures} (replacing the discrete measure $\bm{\mu}_i$'s with the continuous measures $\mu_i$'s),
\begin{equation*}
    W_{\alpha, K, \epsilon}(\mu_i, \mu_j) = \norm{\mP_{\epsilon}^{1/\epsilon} \mu_i - \mP_{\epsilon}^{1/\epsilon} \mu_j }_1 + \sum_{k = 0}^{K - 1}2^{-k\alpha}\norm{(\mP_{\epsilon}^{2^{-(k+1)}/\epsilon} \mu_i - \mP_{\epsilon}^{2^{-k}/\epsilon} \mu_i) - (\mP_{\epsilon}^{2^{-(k+1)}/\epsilon} \mu_j - \mP_{\epsilon}^{2^{-k}/\epsilon}\mu_j)}_1
\end{equation*}

note that by definition of $\widehat{W}_{\mP_{\epsilon}^{t/\epsilon}}$ we have $\lim_{K \to \infty} W_{\alpha, K, \epsilon}(\mu_i, \mu_j) = \widehat{W}_{\mP_{\epsilon}^{t/\epsilon}}(\mu_i, \mu_j)$ and thus 
$$
\lim_{N \to \infty} W_{\alpha, K, \epsilon}(X_i, X_j) = \widehat{W}_{\mP_{\epsilon}^{t/\epsilon}}(\mu_i, \mu_j)
$$ 

combining this result with Corollary~\ref{th: cor_W_using_P_to_H} yields $ W_{\alpha, K, \epsilon}(X_i, X_j) \simeq W_{D_{\alpha}}(\mu_i, \mu_j)$ for $N$ large enough and $\epsilon$ small enough. Now if we take $\epsilon = 2^{-K}$ and apply Theorem~\ref{th: W_DA equiv W_MIN} we get $W_{\alpha, K}(X_i, X_j) \simeq W_{d_{\mathcal{M}}^{2\alpha}}(\mu_i, \mu_j)$ as required.

\end{proof}

\section{Related Work}\label{sec:supp:related}

Wavelet-based linear time approximations of the earth mover's distance have been investigated by \citet{indyk_fast_2003} who used randomly shifted multi-scale grids to compute EMD between images. \citet{shirdhonkar_approximate_2008} expanded this to wavelets over $\mathbb{R}^n$ showing the benefit of using smooth bins over the data, and the relationship to the dual problem. \citet{leeb_holderlipschitz_2016} investigated using wavelets over more general metric spaces. We build off of this theory and efficiently approximate this distance on discrete manifolds represented by sparse kernels.

Another line of work in approximations to the Wasserstein distance instead works with tree metrics over the data~\cite{leeb_mixed_2018, le_tree-sliced_2019, backurs_scalable_2020, sato_fast_2020}. These tree methods also linear time however are built in the ambient dimension and do not represent graph distances as well as graph based methods as shown in Sec.~\ref{sec:results}. Averaging over certain tree families can linked to approximating the EMD with a Euclidean ground distance in $\mathbb{R}^n$~\cite{indyk_fast_2003}. Many trees may be necessary to smooth out the effects of binning over a continuous space as is evident in Figure~\ref{fig:1d}. Diffusion EMD uses multiple scales of smooth bins to reduce this effect, thus giving a smoother distance which is approximately equivalent to a snowflake of the $L^2$ distance on the manifold.

A third line of work considers entropy regularized Wasserstein distances~\cite{cuturi_sinkhorn_2013, benamou_iterative_2014, solomon_convolutional_2015}. These methods show that the transportation plan of the regularized 2-Wasserstein distance is a rescaling of the heat kernel accomplished by the iterative matrix rescaling algorithm known as the Sinkhorn algorithm. In particular, \citet{solomon_convolutional_2015} links the time parameter $t$ in the heat kernel $H_t$ to the entropy regularization parameter and performs Sinkhorn scaling with $H_t$. While applying the heat kernel is efficient, to embed $m$ distributions with the entropy regularized Wasserstein distance is $O(m^2)$ as all pairwise distances must be considered. Next we explore the linear time algorithms based on multiscale smoothing for compute the earth mover's distance.

\subsection{Multiscale Methods for Earth Mover's Distance}\label{sec:supp:multiscale}

Let a (possibly randomized) transformation $T$ map distributions on $\Omega$ to a set of multiscale bins $\vb$, $T: \mu(\Omega) \rightarrow \vb$, these methods define a $T$ such that
\begin{equation}\label{eq:multiscale_form}
W_d(\mu, \nu) \approx \mathbb{E} \|T (\mu) - T (\nu)\|_1.
\end{equation}
Where the approximation, the randomness, and the transform depend on the exact implementation. All of these methods have two things in common, they smooth over multiple scales 

\citet{indyk_fast_2003} presented one of the early methods of this type. They showed that by computing averages over a set of randomly placed grids at dyadic scales. These grids work well for images where the pixels form a discrete set of coordinates.

\citet{shirdhonkar_approximate_2008} showed how to generalize this work over images to more general bin types, more specifically they showed how wavelets placed centered on a grid could replace the averages in \citet{indyk_fast_2003}. This work linked the earth mover's distance in the continuous domain to that of the discrete through wavelets in $\mathbb{R}^d$, showing that for some wavelets the EMD approximation was better than the previous grid method. In Diffusion EMD we generalize these wavelets to the graph domain using diffusion wavelets on the graph allowing Wasserstein distances with a geodesic ground distance.

\citetSupp{kolouri_sliced_2016} in Sliced Wasserstein distances showed how the Wasserstein distance could be quickly approximated by taking the Wasserstein distance along many one dimensional slices of the data. This can be thought of binning along one dimension where the cumulative distribution function is represented along $n$ bins (one for each point) with each bin encompassing one more point than the last.

\citet{le_tree-sliced_2019} generalized the Sliced Wasserstein distance back to trees, giving a new method based on multi-level clustering where the data is partitioned into smaller and smaller clusters where the bins are the clusters. They demonstrated this method on high dimensional spaces where previous methods that had $d^{2^k}$ bins at scale $k$ were inefficiently using bins. By clustering based on the data, their ClusterTree performs well in high dimensions. However, by clustering they lose the convergence to the Wasserstein distance with Euclidean ground distance of previous methods for efficiency in high dimensions.

Diffusion EMD uses smooth diffusion wavelets over the graph which can give geodesic distances unlike previous multiscale methods. Furthermore, Diffusion EMD selects a number of bins that depends on the rank of the data at each dyadic scale, which can lead to smaller representations depending on the data. This is summarized in Table~\ref{tab:multiscale}, which details these differences.

\begin{table}[ht]
\caption{Comparison of Multiscale Methods for Earth Mover's Distance}
    \label{tab:multiscale}
    \begin{center}
\begin{small}
\begin{sc}
\adjustbox{width=\linewidth}{
    \begin{tabular}{lrrrrr}
        \toprule
        Method & Scales & Bins per scale & Data Dependent Centers & Smooth Bins & Geodesic Distances \\
        \midrule
        \citet{indyk_fast_2003} & Dyadic & $d^{2^k}$ & No & No & No  \\
        \citet{shirdhonkar_approximate_2008} & Dyadic & $d^{2^k}$ & No & Yes & No \\
        \citetSupp{kolouri_sliced_2016} & $n$ & 1 & Yes & No & No \\
        \citet{le_tree-sliced_2019} & Specified & $C^k$ & Yes & No & No \\
        Diffusion EMD & Dyadic & Rank Dependent & Yes & Yes & Yes
    \end{tabular}}
    \end{sc}
    \end{small}
    \end{center}
\end{table}

\section{Algorithm Details}\label{sec:supp:algorithm}
In this section we present two algorithms for computing the Diffusion EMD, the Chebyshev approximation method and the interpolative decomposition method. The Chebyshev method is more effective when the number of distributions is relatively small. The interpolative decomposition method is more effective when the number of distributions is large relative to the size of the manifold. We also detail how to subsample the normed space based on the spectrum of $\mP$ which can be approximated quickly.

First we define the approximate rank up to precision $\delta$ of a matrix $A \in \mathbb{R}^{n \times n}$ as:
\begin{equation}
    R_\delta(A) := \# \left \{ i : \sigma_i(A) / \sigma_0(A) \ge \delta \right \}
\end{equation}
Where $\sigma_i(A)$ is the $i$th largest singular value of the matrix $A$. The approximate ranks of dyadic powers of $\mP^{2^k}$ are useful for determining the amount of subsampling to do at each scale either after an application of the Chebyshev method or during the computation of (approximate) dyadic powers of $\mP$. We note that based on the density of singular values of $\mP$, which is quickly computable using the algorithm presented in \citet{dong_network_2019}, the approximate rank of all dyadic powers of $\mP$ can be calculated without computing powers of $\mP$.

\paragraph{Chebyshev Approximation of Diffusion EMD}

We first note that $\mP^{2^k}$ can be computed spectrally using a filter on its eigenvalues. Let $\mM = \mD^{1/2} \mP \mD^{-1/2}$ be the symmetric conjugate of $\mP$. Then $\mM$ is symmetric and has eigenvalues lying in the range $-1 \le \lambda_0 \le \lambda_1 \le \ldots \le \lambda_n \le 1$. $\mM$ can be decomposed into $\mU \mSigma \mU^T$ for orthonormal $\mU$ and $\mSigma$ a diagonal matrix of $[\lambda_0, \lambda_1, \ldots, \lambda_n]$. We then express $\mP^{2^k}$ as a filter on the eigenvalues of the $\mM$:
\begin{equation}
    \mP^{2^k} = \mD^{-1/2} \mU (\mSigma)^{2^k} \mU^T \mD^{1/2}
\end{equation}
We compute the first $J$ Chebyshev polynomials of $\mP^j \vmu$ then use the polynomial filter on the eigenvalues $h(\sigma) = \sigma^{2^k}$ to compute diffusions of the distributions, and reweight these diffusion bins as specified in \Algref{alg:cheb_embedding}. This algorithm requires $J$ sparse matrix multiplications of $\mP$ and $\vmu$. For a total time complexity of $\tilde{O}(J m n)$ when $\mP$ is sparse.



\paragraph{Interpolative Decomposition Approximation of Diffusion EMD}

Our second method proceeds by directly approximating multiplication by the matrix $\mP^{2^k}$. The naive solution of computing dyadic powers of $\mP$ on the original basis quickly leads to a dense $n \times n$ matrix. \citet{coifman_diffusion_2006-1} observed that $\mP^{2^k}$ is of low rank, and therefore multiplication by $\mP^{2^k}$ can be approximated on a smaller basis. In that work they introduced a method of iteratively reducing the size of the basis using rank-revealing pivoted sparse QR decomposition, then computing $\mP^{2^{k+1}} = \mP^{2^k} \mP^{2^k}$. By reducing the size of the basis and orthogonalizing the basis is kept small and $\mP^{2^k}$ is kept sparse on its corresponding basis. For sparse $\mP$ this gives an algorithm that is $O(n \log^2 n)$ for computing dyadic powers of a sparse matrix.

Our algorithm follows the same lines as theirs however we use interpolative decomposition to reduce the size of the basis, and do not do so at early stages where the size of the basis may still be a significant fraction of $n$. Interpolative decomposition allows for an easy interpretation: we are selecting a well spaced set of points that acts as a basis for higher levels of $\mP^{2^k}$. For more details on interpolative decomposition we refer the reader to \citet{liberty_randomized_2007, bermanis_multiscale_2013}. The algorithm in \citet{coifman_diffusion_2006-1} also computes powers of $\mP^{2^k}$ in the original basis by projecting the multiplication in the reduced basis back to the original. We do not bother computing $\mP^{2^k}\vmu$ in the original basis, but instead use the reduced basis directly to compare distributions against. Denote the basis at level $k$ as $\phi_k$, which is really a subset of $\phi_{k-1}$, then $\phi_k$ can be thought of as the centers of the bins at level $k$ of our multiscale embedding. We only compare distributions at these representative centers rather than at all datapoints. This is summarized in \Algref{alg:id_embedding}, creating an embedding for each distribution whose length is dependent on the rank of $\mP$.

\paragraph{Sampling the Diffusions at Larger Scales}
Interpolative decomposition is an integral part of \Algref{alg:id_embedding}. However, it can also be useful in reducing the size of the distribution representations of \Algref{alg:cheb_embedding} which uses the Chebyshev approximation. Without sampling the Chebyshev algorithm centers a bin at every datapoint at every scale. However, for larger scales this is unnecessary, and we can take advantage of the low rank of $\mP^{2^k}$ with interpolative decomposition. In practice, we use the top 6 largest scales, in fact the distance shows very little sensitivity to small scales (see Fig.~\ref{fig:ablation}(d)). In fact, this representation can be compressed without a significant loss in performance as shown in Fig.~\ref{fig:ablation}(b,c). We apply the rank threshold to the spectrum of $\mP^{2^k}$, which can be computed from the density of states algorithm described above, to determine the number and location of centers to keep. With $\delta = 10^{-6}$ this reduces the number of centers from $60,000$ to $<4,000$ with a small loss in performance. Rather than selecting a fixed number of centers per scale as in previous methods, this allows us to vary the number of centers per scale as necessary to preserve the rank at each scale.

\paragraph{Time Complexity} Here, we assume that $\mP$ has at most $\tilde{O}(n) = O(n \log^c n)$ nonzero entries. In this case, interpolative decomposition in \Algref{alg:id} has complexity $O(k \cdot m \cdot n \log n)$, and the randomized version in \Algref{alg:rand_id} has complexity $O(k m + k^2 n)$~\cite{liberty_randomized_2007}. \Algref{alg:cheb_embedding} has complexity $O(J m n \log^c n)$ which is dominated by calculation of the $J$ powers of $\mP^j \vmu$. The computation time of \Algref{alg:id_embedding} is dominated by the first randomized interpolative decomposition step. When we set $\gamma = O(\log^c n)$, which controls when the application of RandID occurs, this gives a total time complexity for \Algref{alg:id_embedding} of $\tilde{O}(m n) = O(n \log^{2c} n + m n \log^b n )$, where $b$ and $c$ are constants controlled by $\gamma$. As $\gamma$ increases, $b$ increases and $c$ decreases. The first term is dominated by the first interpolative decomposition setup, and the second is the calculation of $\vmu^{(2^k)}$. It follows that when $m \gg n$, it is helpful to set $\gamma$ larger for a better tradeoff of $b$ and $c$. In practice we set $\gamma = n / 10$ which seems to offer a reasonable tradeoff between the two steps. For small $m$ in practice \Algref{alg:cheb_embedding} is significantly faster than \Algref{alg:id_embedding}.

\begin{algorithm}[tb]
    \caption{Interpolative Decomposition}
    \label{alg:id}
\begin{algorithmic}
    \STATE {\bfseries Input:} $m \times n$ matrix $A$ and number of subsamples $k$ s.t. $k < \min\{m,n\}$.
    \STATE {\bfseries Output:} $m \times k$ matrix $B$ and $k \times n$ matrix $C$ s.t. $\|AP - BC\|_2 \le \sqrt{4k(n-k) + 1} \sigma_{k+1}(A)$, where $B$ consists of $k$ columns of $A$. 
    \STATE Perform pivoted QR decomposition on $A$ s.t. $AP = QR$, where $P$ is a permutation matrix of $A$.
    \STATE Where Q and R are decomposed as the following with $Q_1$ is a $m \times k$ matrix, $Q_2$ is $m \times (n - k)$, etc. 
    \begin{equation*}
        Q = [Q_1 \mid Q_2]; \quad\quad
        R = \begin{bmatrix}
            R_{11} & R_{12} \\ 
            \mathbf{0}_{k \times n-k} & R_{22} 
            \end{bmatrix}
    \end{equation*}
    \STATE $B \leftarrow Q_1 R_{11}$
    \STATE $C \leftarrow [I_k | R_{11}^{-1} R_{12}]$
\end{algorithmic}
\end{algorithm}

\begin{algorithm}[tb]
    \caption{Randomized Interpolative Decomposition: $\text{RandID}(\mA, j, k, l)$}
    \label{alg:rand_id}
\begin{algorithmic}
    \STATE {\bfseries Input:} $m \times n$ matrix $A$, random sample sizes $j, l$, and number of subsamples $k$ s.t. $\min\{m,n\} > j > k > l$.
    \STATE {\bfseries Output:} $m \times k$ matrix $B$ and $k \times n$ matrix $C$ s.t. $\|A - BC\|_2 \le \sqrt{4k(n-k) + 1} \sigma_{k+1}(A)$ w.h.p., where $B$ consists of $k$ columns of $A$.
    \STATE $G_1 \leftarrow \mathcal{N}(0, 1)^{j \times m}$
    \STATE $W \leftarrow G_1 A$ 
    \STATE Perform the (deterministic) interpolative decomposition on $W$ in \Algref{alg:id} to get $\tilde{B}, C$ s.t.  $\|W - \tilde{B}D\|_2$ small.
    \STATE $G_2 \leftarrow \mathcal{N}(0,1)^{l \times j}$
    \STATE $W_2 \leftarrow G_2 W$ 
    \STATE $\tilde{B}_2 \leftarrow G_2 \tilde{B}$
    \STATE Find the $k$ indices of $\tilde{B}_2$ that are approximately equal to columns in $W_2$, $[i_1, i_2, \ldots, i_k]$.
    \STATE $\vb \leftarrow [\vb_1, \vb_2, \ldots, \vb_K]$
\end{algorithmic}
\end{algorithm}


\begin{algorithm}[tb]
    \caption{Interpolative Decomposition diffusion densities $T_{id}(\mK, \vmu, K, \alpha, \delta, \gamma)$}
    \label{alg:id_embedding}
\begin{algorithmic}
    \STATE {\bfseries Input:} $n \times n$ graph kernel $\mK$, $n \times m$ distributions $\vmu$, maximum scale $K$, snowflake constant $\alpha$, rank threshold $\delta$, and basis recomputation threshold $\gamma$.
    \STATE {\bfseries Output:} Distribution embeddings $\vb$
    \STATE $\mQ \leftarrow Diag(\sum_i \mK_{ij})$
    \STATE $\mK^{norm} \leftarrow \mQ^{-1} \mK \mQ^{-1}$
    \STATE $\mD \leftarrow Diag(\sum_i \mK^{norm}_{ij})$
    \STATE $\mM \leftarrow \mD^{-1/2}\mK^{norm}\mD^{-1/2}$
    \STATE $\vmu^{(2^0)} \leftarrow \mP \vmu \leftarrow \mD^{-1/2}\mM\mD^{1/2} \vmu$
    \STATE $\vmu_0 \leftarrow \vmu$
    \STATE $n_0 \leftarrow n$
    \FOR{$k=1$ {\bfseries to} $K$}
        \IF {$R_\delta(\mM^{2^k}) < \gamma $}
            \STATE $\mB_k, \mC_k \leftarrow \text{RandID}(\mM_{k-1}, R_\delta(\mM^{2^k})+8, R_\delta(\mM^{2^k}), 5)$
            \STATE $n_k \leftarrow R_\delta(\mM^{2^k})$
        \ELSE
            \STATE $\mB_k, \mC_k, n_k \leftarrow \mB_{k-1}, \mI_{n_{k-1}}, n_{k-1} $
        \ENDIF
        \STATE $\mM_k \leftarrow \mC_k \mM_{k-1} \mM_{k-1}^T \mC_k^T$
        \STATE $\vmu_k \leftarrow \mC_k \vmu_{k-1}$
        \STATE $\vmu^{(2^k)} \leftarrow \mD^{-1/2} \mM_k  \mD^{1/2} \vmu_k$
        \STATE $\vb_{k-1} \leftarrow 2^{(K-k-1)\alpha}(\vmu^{(2^k)} - \mC_k \vmu^{(2^{k-1})})$ 
    \ENDFOR
    \STATE $\vb_{K} \leftarrow \vmu^{(2^K)}$
    \STATE $\vb \leftarrow [\vb_0, \vb_1, \ldots, \vb_K]$
\end{algorithmic}
\end{algorithm}

\subsection{Gradients of the Diffusion EMD}\label{subsec:gradient}
In this section we will  develop the computation of the gradient of the Diffusion EMD $W_{\alpha, K}(X_i, X_j)$, we will show that its gradient depend only on the gradient of the Gaussian kernel matrix $\mK_\epsilon$ which is easy to compute. We start by recalling some basic facts about the gradient of matrices,
\begin{enumerate}
    \item For a matrix $P$ and a vector $\mu$ we have $\nabla \norm{P \mu}_1 = \mathrm{sign}(P \mu) \mu^T \nabla P $. 
    \item For an invertible matrix $P$, the gradient of the inverse of $P$ is $\nabla (P^{-1}) = -P^{-1} \nabla(P) P^{-1}$.
    \item Let $n\in \mathbb{N}^*$ the gradient of $P^n$ is 
    \begin{equation*}
        \nabla P^n = \sum_{\ell = 0}^{n - 1} P^{\ell} \nabla(P) P^{n - \ell -1}.
    \end{equation*}
\end{enumerate}
Using the definition of the Diffusion EMD, we have
\begin{align*}
    &\nabla W_{\alpha, K}(X_i, X_j) = \nabla \norm{\mP_{\epsilon}^{2^K}(\bm{\mu}_j - \bm{\mu}_i)}_1 + \sum_{k = 0}^{K - 1}2^{-(K-k-1)\alpha}\nabla \norm{\left(\mP_{\epsilon}^{2^{k+1}} - \mP_{\epsilon}^{2^k}\right)(\bm{\mu_j} - \bm{\mu}_i)}_1 \\ 
    &= \mathrm{sign} \left( \mP_{\epsilon}^{2^K}(\bm{\mu}_j - \bm{\mu}_i) \right) (\bm{\mu}_j - \bm{\mu}_i)^T \nabla(\mP_{\epsilon}^{2^K}) + \sum_{k = 0}^{K - 1}2^{-(K-k-1)\alpha} \mathrm{sign}\left(\mP_{\epsilon}^{2^{k+1}} - \mP_{\epsilon}^{2^k}\right) (\bm{\mu_j} - \bm{\mu}_i))^T \nabla(\mP_{\epsilon}^{2^{k+1}} - \mP_{\epsilon}^{2^k}). \\ 
\end{align*}
The last formula tells us that the gradient of $W_{\alpha, K}(X_i, X_j)$ depends only on the gradient of the powers of $\mP_\epsilon$ which in turn can be expressed in terms of the gradients $\mP_{\epsilon}$. We have:
\begin{align*}
    \nabla P_\epsilon &= \nabla(D^{-1}_{\epsilon} M_{\epsilon} D^{-1}_{\epsilon}) \\
    &= \nabla (D^{-1}_{\epsilon})M_{\epsilon} + D_{\epsilon}^{-1}\nabla(M_{\epsilon}) \\
    &= -D_{\epsilon}^{-1}\nabla(D_{\epsilon})D_{\epsilon}^{-1} M_{\epsilon} + D_{\epsilon}^{-1}\nabla(M_{\epsilon}) \\ 
    &= -D_{\epsilon}^{-1}\nabla(\mathrm{diag}(M_{\epsilon} \mathbf{1}))D_{\epsilon}^{-1} M_{\epsilon} + D_{\epsilon}^{-1}\nabla(M_{\epsilon}) \\
    &= -D_{\epsilon}^{-1} \mathrm{diag}(\nabla(M_{\epsilon}) \mathbf{1}) D_{\epsilon}^{-1} M_{\epsilon} + D_{\epsilon}^{-1}\nabla(M_{\epsilon}), \\
\end{align*}
therefore in order to compute $\nabla P_\epsilon$ we need to compute $\nabla M_\epsilon$,
\begin{align*}
    \nabla M_{\epsilon} &= \nabla(Q_{\epsilon}^{-1} K_{\epsilon} Q_{\epsilon}^{-1}) \\
    & = \nabla(Q_{\epsilon}^{-1}) K_{\epsilon} Q_{\epsilon}^{-1} + Q_{\epsilon}^{-1} \nabla(K_{\epsilon}) Q_{\epsilon}^{-1} + Q_{\epsilon}^{-1} K_{\epsilon} \nabla(Q_{\epsilon}^{-1})\\
    & = -Q_{\epsilon}^{-1}\nabla(Q_{\epsilon}) Q_{\epsilon}^{-1} K_{\epsilon} Q_{\epsilon}^{-1}
    + Q_{\epsilon}^{-1}\nabla(K_\epsilon) Q_{\epsilon}^{-1}
    - Q_{\epsilon}^{-1}K_{\epsilon} Q_{\epsilon}^{-1} \nabla(Q_{\epsilon})  Q_{\epsilon}^{-1} \\ 
    & = -Q_{\epsilon}^{-1}\nabla(\mathrm{diag}(K_{\epsilon} \mathbf{1})) Q_{\epsilon}^{-1} K_{\epsilon} Q_{\epsilon}^{-1}
    + Q_{\epsilon}^{-1}\nabla(K_\epsilon) Q_{\epsilon}^{-1}
    - Q_{\epsilon}^{-1}K_{\epsilon} Q_{\epsilon}^{-1} \nabla(\mathrm{diag}(K_{\epsilon} \mathbf{1}))  Q_{\epsilon}^{-1} \\ 
    & = -Q_{\epsilon}^{-1}\mathrm{diag}(\nabla(K_{\epsilon}) \mathbf{1}) Q_{\epsilon}^{-1} K_{\epsilon} Q_{\epsilon}^{-1}
    + Q_{\epsilon}^{-1}\nabla(K_\epsilon) Q_{\epsilon}^{-1}
    - Q_{\epsilon}^{-1}K_{\epsilon} Q_{\epsilon}^{-1} \mathrm{diag}(\nabla(K_{\epsilon}) \mathbf{1})  Q_{\epsilon}^{-1}
\end{align*}

hence the gradient of $M_{\epsilon}$ can be written only in terms of the gradient of $K_\epsilon$. This process shows that the gradient of the EMD diffusion can be expressed in terms of the gradient of the Gaussian kernel only, which may make it useful in future applications where fast gradients of the earth mover's distance are necessary. For example, in distribution matching where previous methods use gradients with respect to the Sinkhorn algorithm~\cite{frogner_learning_2015}, which scales with $O(n^2)$. In this application, computation of gradients of Diffusion EMD would be $O(n)$. What makes this possible is a smooth kernel and smooth density bins over the graph. In contrast to Diffusion EMD, multiscale methods that use non-smooth bins have non-smooth gradients with respect to $K_\epsilon$, and so are not useful for gradient descent type algorithms. We leave further investigation into the gradients of Diffusion EMD to future work. 

\section{Experiment details}\label{sec:supp:experiments}

We ran all experiments on a 36 core Intel(R) Xeon(R) CPU E5-2697 v4 @ 2.30GHz with 512GB of RAM. We note that our method is extremely parallelizable, consisting of only matrix operations, thus a GPU implementation could potentially speed up computation. For fair comparison we stick to CPU implementations and leave GPU acceleration to future work. We provide custom implementations of all methods in the \texttt{DiffusionEMD} package, which is available on Github. We base our Sinkhorn implementation on that in the python optimal transport package (POT)~\citeSupp{flamary_pot_2021}, and our tree methods off of the sklearn nearest neighbors trees~\cite{pedregosa_scikit-learn:_2011}.

\subsection{Metrics used}
\paragraph{P@10.} P@10 is a metric that measures the overlap in the 10-nearest neighbors between the nearest neighbors of the EMD method and the ground truth nearest neighbors, which is calculated using exact OT on the geodesic ground distance. P@10 is useful for distinguishing the quality of distance calculations locally, whereas the next metric measures the quality more globally. This may be particularly important in applications where only the quality of nearest neighbors matters as in Sect~\ref{sec:methods:embeddings}, or in applications of Wasserstein nearest neighbor classification~\cite{indyk_fast_2003, backurs_scalable_2020}. 

\paragraph{Spearman-$\rho$.} The Spearman's rank correlation coefficient (also known as Spearman-$\rho$) measures the similarity in the order of two rankings. This is useful in our application because while the 1-Wasserstein and 2-Wasserstein may have large distortion as metrics, they will have a high Spearman-$\rho$ in all cases. Spearman-$\rho$ is defined for two rankings of the data where $d_i$ is the difference in the rankings of each for each datapoint $i$ as
\begin{equation*}
    \rho = 1 - \frac{6 \sum d_i^2}{n (n^2 - 1)}
\end{equation*}
This is useful for quantifying the global ranking of distributions in the graph, which shows the more global behavior of EMD methods. Since locally manifolds with low curvature look very similar to $\mathbb{R}^d$, where $d$ is the intrinsic dimension of the manifold, for distributions very close together on the manifold using a Euclidean ground distance may be acceptable, however for distributions far away from each other according to a geodesic ground distance, approximation with a Euclidean ground distance may perform poorly.

\subsection{Line Example}

In Fig~\ref{fig:1d} we used a cluster tree with 4 levels of 4 clusters for the cluster tree, and a quad tree of depth 4. These are mostly for illustrative purposes so we there was no parameter search for ClusterTree or QuadTree. As the number of trees gets larger ClusterTree and QuadTree start to converge to some smooth solution. ClusterTree has a number of bumps caused by the number of clusters parameter that do not disappear even with many trees. Quadtree converges to a smooth solution as the number of trees gets large in this low dimensional example.

\subsection{Swiss Roll}

On the swiss roll example we compared Diffusion EMD to ClusterTree, QuadTree, and the convolutional Sinkhorn method. Here we describe the parameter setup in each of these methods for Figure~\ref{fig:swissroll_embed}, \ref{fig:swissroll}, and \ref{fig:ablation}. 

For Fig.~\ref{fig:swissroll_embed} we chose parameters that approximately matched the amount of time between methods on this dataset to the Diffusion EMD with default settings using the Chebyshev algorithm. This was 150 cluster trees of depth 8 with 3 cluster, and 20 quad trees of depth 4. We noticed that the speed of quadtree diminishes exponentially with dimension. 

For Fig.~\ref{fig:swissroll} we did a grid search over parameters for QuadTree and Cluster Tree to find the best performing settings of depth and number of clusters. We searched over depth $d \in [2 \ldots 10]$ and number of clusters $C \in [2 \ldots 5]$ for ClusterTree. We found a depth of 4 for QuadTree and 3 Clusters of depth 5 for ClusterTree gave the best tradeoff of performance vs. time on this task. These parameters were tried while fixing the number of trees to 10. We fixed this depth and number of clusters for subsequent experiments varying the number of trees in the range $[1 \ldots 100]$ for QuadTree and in the range $[1 \ldots 15]$ for ClusterTree. The maximum of this range exceeded the time allotment of Diffusion EMD in both cases. For the convolutional Sinkhorn method we fixed parameter $t=50$ for $H_t$, and fixed the maximum number of iterations in the range $[1 \ldots 100]$. This took orders of magnitude longer than the multiscale methods in all cases, with comparable performance to Diffusion EMD on the P@10 metric and better performance using the Spearman-$\rho$ metric. We note that a similar multi-step refinement scheme as implemented in \citet{backurs_scalable_2020} could be used here.

\subsection{MNIST}
To construct the Spherical MNIST dataset~\cite{cohen_convolutional_2017}, we projected the 70,000 images in classical MNIST onto a spherical graph containing 25,088 nodes. Every pixel in each MNIST image was projected onto four nodes on the northern hemisphere. These projections were treated as signals over a spherical graph.

We ran Diffusion EMD, ClusterTree, and QuadTree on this dataset to generate embeddings in L1 space. Diffusion EMD was run with the diffusion wavelet approximations described in the paper, using an epsilon value of $0.001$. To best compare accuracy, we chose parameters for QuadTree and ClusterTree that approximately matched the runtime of Diffusion EMD. For QuadTree, we used a depth of 20 and 10 clusters. For ClusterTree, we used a depth of 5. However, we note that both ClusterTree and QuadTree experienced a variance of 10-30 minutes between subsequent runtimes with the same parameters, probably due to differences of random initialization. We did not run Convolutional Sinkhorn on this dataset, as the projected runtime exceeded 24 hours.

We calculated a ground truth EMD between 100 of the projected MNIST digits using the exact EMD implementation of Python Optimal Transport. More accurate comparisons might be obtained using more points, but computation of this 100 by 100 distance matrix took over 24 hours, necessitating severe subsampling, and once again highlighting the need for fast and accurate approximations of EMD.

For each algorithm, we obtained an accuracy score by running a kNN classifier on the embeddings using an even train/test split considering only the nearest neighbors. We then computed the Spearman $\rho$ and P@10 score between each method and the exact EMD ground truth.

\subsection{Single cell COVID-19 Dataset}

One hundred sixty eight patients with moderate to severe COVID-19 \citeSupp{Marshall2020} were admitted to YNHH and recruited to the Yale Implementing Medical and Public Health Action Against Coronavirus CT (IMPACT) study. From each patient, blood samples were collected to characterize patient cellular responses across timepoints to capture the spectrum of disease. In total, the composition of peripheral blood mononuclear cell (PBMC) was measured by a myeloid-centric flow cytometry on 210 samples. Finally, clinical data was extracted from the electronic health record corresponding to each biosample timepoint to allow for clinical correlation of findings as shown previously. In the main analysis, poor or adverse outcomes were defined by patient death, while good outcomes were defined by patient survived. In order to analyze the 22 million cells generated in this study, we performed k-means clustering, setting k to 27,000, and took the resultant cluster centroids as a summarization of the cellular state space as done previously in  \cite{kuchroo_multiscale_2020}.

To test the quality of the organization of patients using EMD methods, we test the Laplacian smoothness on the 10-nearest neighbors graph created for each method. The Laplacian smoothness of a function on the nodes is denoted $f^T L f$ where $L = D - A$ is the (combinatorial) Laplacian of the adjacency matrix $A$. This is equivalent to computing the following sum:
\begin{equation*}
    \sum_{(i, j) \in E} (f_i - f_j)^2
\end{equation*}
where $E$ is the edge set of the kNN graph. This measures the sum of squared differences of $f$ along the edges of the kNN graph, which measures the smoothness of $f$ over the graph.

Various cell types have been shown to be associated with mortality from COVID-19 infection, including CD16$^{+}$ Neutrophils, T cells and non-classical monocytes. Previously, T cells-to-CD16$^{+}$ Neutrophils ratios have been reported to be predict of outcome, with low ratios predicting mortality and high ratios predicting survival \citeSupp{Li2020}. Even outside of ratios, the absolute counts to T cells \citeSupp{Chen2020} and CD16$^{+}$ Neutrophils \citeSupp{Meizlish2020.09.01.20183897} have been shown to be associated with mortality. Finally, non-classical monocytes have also been shown to be enriched in patients with more severe outcomes \citeSupp{Pence2020}, further supporting our findings.

\bibliographystyleSupp{icml2021}
\bibliographySupp{supp}

\end{document}